%% file: main.tex
\PassOptionsToPackage{table}{xcolor}
\documentclass{article} 
\usepackage{referencepaper}

\input{math_commands.tex}

\usepackage{amsthm}
\theoremstyle{plain}
\newtheorem{proposition}{Proposition}
\newtheorem{corollary}{Corollary}

\usepackage{booktabs,tabularx,multirow,array,makecell,float}
\usepackage{placeins}
\usepackage{graphicx}
\usepackage{xcolor}
\usepackage{xspace}
\usepackage{pifont}
\usepackage{enumitem}
\floatstyle{ruled}
\newfloat{algorithm}{H}{loa}
\floatname{algorithm}{Algorithm}
\floatstyle{plain}
\usepackage{hyperref}
\usepackage{url}
\hypersetup{hidelinks}

\definecolor{MethodBlue}{HTML}{EDF5FC}
\definecolor{GainBlue}{HTML}{DCECF8}
\definecolor{GroupGray}{HTML}{F2F3F5}

\newcommand{\BaseOvosUniform}{48.98}
\newcommand{\BaseOvosSupport}{54.08}

\newcommand{\BaseOvosLone}{54.07}
\newcommand{\BaseOvosLtwo}{57.21}
\newcommand{\BaseOvosLthree}{49.48}
\newcommand{\BaseOvosLfour}{35.16}
\newcommand{\BaseOst}{50.52}
\newcommand{\BaseUcs}{45.33}
\newcommand{\BaseMean}{48.28}

\newcommand{\CsaOvosUniform}{53.89}
\newcommand{\CsaOvosSupport}{57.98}
\newcommand{\CsaOvosPolicyGain}{4.09}
\newcommand{\CsaOvosUniformLone}{58.60}
\newcommand{\CsaOvosUniformLtwo}{61.68}
\newcommand{\CsaOvosUniformLthree}{55.11}
\newcommand{\CsaOvosUniformLfour}{40.17}
\newcommand{\CsaOvosUniformFOneOne}{62.87}
\newcommand{\CsaOvosUniformFOneTwo}{64.43}
\newcommand{\CsaOvosUniformFOneThree}{48.51}
\newcommand{\CsaOvosUniformFTwoOne}{59.09}
\newcommand{\CsaOvosUniformFTwoTwo}{51.79}
\newcommand{\CsaOvosUniformFTwoThree}{74.15}
\newcommand{\CsaOvosUniformFThreeOne}{65.33}
\newcommand{\CsaOvosUniformFThreeTwo}{33.33}
\newcommand{\CsaOvosUniformFThreeThree}{66.67}
\newcommand{\CsaOvosUniformFFourOne}{24.43}
\newcommand{\CsaOvosUniformFFourTwo}{67.78}
\newcommand{\CsaOvosUniformFFourThree}{28.30}
\newcommand{\CsaOvosLone}{67.26}
\newcommand{\CsaOvosLtwo}{65.10}
\newcommand{\CsaOvosLthree}{57.26}
\newcommand{\CsaOvosLfour}{42.30}
\newcommand{\CsaOvosFOneOne}{69.46}
\newcommand{\CsaOvosFOneTwo}{71.13}
\newcommand{\CsaOvosFOneThree}{61.19}
\newcommand{\CsaOvosFTwoOne}{60.10}
\newcommand{\CsaOvosFTwoTwo}{58.33}
\newcommand{\CsaOvosFTwoThree}{76.87}
\newcommand{\CsaOvosFThreeOne}{68.67}
\newcommand{\CsaOvosFThreeTwo}{34.92}
\newcommand{\CsaOvosFThreeThree}{68.18}
\newcommand{\CsaOvosFFourOne}{26.72}
\newcommand{\CsaOvosFFourTwo}{70.00}
\newcommand{\CsaOvosFFourThree}{30.19}
\newcommand{\CsaOst}{63.93}
\newcommand{\CsaUcs}{52.77}
\newcommand{\CsaOvosTrainGain}{4.91}
\newcommand{\CsaOstTrainGain}{13.41}
\newcommand{\CsaUcsTrainGain}{7.44}

\newcommand{\ReplicaOvosUniform}{54.10}
\newcommand{\ReplicaOvosSupport}{58.00}
\newcommand{\ReplicaOvosPolicyGain}{3.90}
\newcommand{\ReplicaOst}{63.24}
\newcommand{\ReplicaUcs}{54.26}
\newcommand{\ReplicaMean}{57.20}
\newcommand{\HeadlineMean}{56.86}

\newcommand{\CommonInterfaceOst}{64.72}
\newcommand{\CommonInterfaceUcs}{51.34}
\newcommand{\CommonInterfaceMean}{58.03}
\newcommand{\NativeInterfaceMean}{58.75}

\newcommand{\BaseOstAState}{48.28}
\newcommand{\BaseOstAO}{36.36}
\newcommand{\BaseOstAInfo}{68.96}
\newcommand{\BaseOstEst}{33.80}
\newcommand{\BaseOstJudge}{60.70}
\newcommand{\BaseOstCount}{61.37}
\newcommand{\BaseOstTemporal}{44.86}
\newcommand{\CsaOstAState}{53.83}
\newcommand{\CsaOstAO}{51.78}
\newcommand{\CsaOstAInfo}{84.69}
\newcommand{\CsaOstEst}{43.63}
\newcommand{\CsaOstJudge}{72.72}
\newcommand{\CsaOstCount}{78.54}
\newcommand{\CsaOstTemporal}{62.45}

\newcommand{\BaseOstPosJudge}{58.65}
\newcommand{\BaseOstPosEst}{42.12}
\newcommand{\BaseOstOriJudge}{55.32}
\newcommand{\BaseOstOriEst}{37.02}
\newcommand{\BaseOstExistJudge}{83.29}
\newcommand{\BaseOstExistTemporal}{64.07}
\newcommand{\BaseOstQuantity}{61.37}
\newcommand{\BaseOstDiversity}{76.33}
\newcommand{\BaseOstOrder}{59.71}
\newcommand{\BaseOstDirJudge}{43.35}
\newcommand{\BaseOstDirTemporal}{25.57}
\newcommand{\BaseOstDirEst}{21.58}
\newcommand{\BaseOstDistJudge}{48.27}
\newcommand{\BaseOstDistTemporal}{44.93}
\newcommand{\BaseOstDistEst}{34.47}
\newcommand{\CsaOstPosJudge}{65.82}
\newcommand{\CsaOstPosEst}{45.00}
\newcommand{\CsaOstOriJudge}{61.30}
\newcommand{\CsaOstOriEst}{43.20}
\newcommand{\CsaOstExistJudge}{94.28}
\newcommand{\CsaOstExistTemporal}{82.40}
\newcommand{\CsaOstQuantity}{78.54}
\newcommand{\CsaOstDiversity}{87.92}
\newcommand{\CsaOstOrder}{80.29}
\newcommand{\CsaOstDirJudge}{60.06}
\newcommand{\CsaOstDirTemporal}{41.67}
\newcommand{\CsaOstDirEst}{48.45}
\newcommand{\CsaOstDistJudge}{59.34}
\newcommand{\CsaOstDistTemporal}{63.29}
\newcommand{\CsaOstDistEst}{37.89}
\newcommand{\DeltaOstPosJudge}{+7.17}
\newcommand{\DeltaOstPosEst}{+2.88}
\newcommand{\DeltaOstOriJudge}{+5.97}
\newcommand{\DeltaOstOriEst}{+6.18}
\newcommand{\DeltaOstExistJudge}{+10.99}
\newcommand{\DeltaOstExistTemporal}{+18.33}
\newcommand{\DeltaOstQuantity}{+17.17}
\newcommand{\DeltaOstDiversity}{+11.59}
\newcommand{\DeltaOstOrder}{+20.57}
\newcommand{\DeltaOstDirJudge}{+16.71}
\newcommand{\DeltaOstDirTemporal}{+16.09}
\newcommand{\DeltaOstDirEst}{+26.86}
\newcommand{\DeltaOstDistJudge}{+11.07}
\newcommand{\DeltaOstDistTemporal}{+18.36}
\newcommand{\DeltaOstDistEst}{+3.42}

\newcommand{\BaseUcsEgo}{41.72}
\newcommand{\BaseUcsTrajectory}{48.49}
\newcommand{\BaseUcsObject}{43.87}
\newcommand{\BaseUcsProximity}{49.49}
\newcommand{\CsaUcsEgo}{48.26}
\newcommand{\CsaUcsTrajectory}{56.14}
\newcommand{\CsaUcsObject}{53.17}
\newcommand{\CsaUcsProximity}{54.35}
\newcommand{\CsaUcsMyTrajectory}{54.87}
\newcommand{\CsaUcsObjectTrajectory}{59.95}
\newcommand{\CsaUcsEgoPosition}{46.77}
\newcommand{\CsaUcsRelativePose}{49.21}
\newcommand{\CsaUcsDistance}{51.71}
\newcommand{\CsaUcsReachability}{56.43}
\newcommand{\CsaUcsObjectCategory}{59.59}
\newcommand{\CsaUcsQuantity}{50.33}

\newcommand{\ReplayRows}{39{,}561}
\newcommand{\PreservationRows}{33{,}180}
\newcommand{\TrajectoryRows}{6{,}381}
\newcommand{\TrajectoryTurns}{45{,}006}
\newcommand{\TrajectoryMass}{0.5}
\newcommand{\ShiftedOst}{61.51}
\newcommand{\ChronoSeedTwoOst}{64.52}
\newcommand{\ShiftedSeedTwoOst}{62.11}
\newcommand{\OriginGapSeedOne}{2.42}
\newcommand{\OriginGapSeedTwo}{2.41}
\newcommand{\ChronoTwoSeedMean}{64.22}
\newcommand{\ShiftedTwoSeedMean}{61.81}
\newcommand{\OrderGain}{2.41}

\newcommand{\MeanFrames}{80.50}

\newcommand{\EqualBudgetOvos}{52.40}

\newcommand{\AdaptiveEqualBudgetGain}{5.58}

\newcommand{\CategoricalControlOvos}{58.03}
\newcommand{\CategoricalControlOst}{47.48}
\newcommand{\CategoricalControlUcs}{46.72}
\newcommand{\CategoricalControlMean}{50.74}
\newcommand{\TrajectoryControlOvos}{52.03}
\newcommand{\TrajectoryControlOst}{51.03}
\newcommand{\TrajectoryControlUcs}{49.22}
\newcommand{\TrajectoryControlMean}{50.76}

\newcommand{\CsaOvosLfourGain}{5.01}         
\newcommand{\CommonInterfaceOstDelta}{1.48}  
\newcommand{\NativeInterfaceUcsDelta}{2.92}  

\newcommand{\DeltaUcsTrajectory}{+7.65}       
\newcommand{\DeltaUcsMyTrajectory}{+7.38}     
\newcommand{\DeltaUcsObjectTrajectory}{+8.46} 
\newcommand{\DeltaUcsEgo}{+6.54}              
\newcommand{\DeltaUcsEgoPosition}{+6.45}      
\newcommand{\DeltaUcsRelativePose}{+6.60}     
\newcommand{\DeltaUcsProximity}{+4.86}        
\newcommand{\DeltaUcsDistance}{+5.59}         
\newcommand{\DeltaUcsReachability}{+4.29}     
\newcommand{\DeltaUcsObject}{+9.30}           
\newcommand{\DeltaUcsObjectCategory}{+6.09}   
\newcommand{\DeltaUcsQuantity}{+10.73}        

\newcommand{\TrainLone}{+4.54}
\newcommand{\TrainLtwo}{+4.46}
\newcommand{\TrainLthree}{+5.63}
\newcommand{\TrainLfour}{+5.01}
\newcommand{\AccessLone}{+8.66}
\newcommand{\AccessLtwo}{+3.43}
\newcommand{\AccessLthree}{+2.15}
\newcommand{\AccessLfour}{+2.13}

\newcommand{\LocalCompatOvosUniform}{53.73}
\newcommand{\LocalCompatOst}{64.08}
\newcommand{\LocalCompatUcs}{52.65}
\newcommand{\LocalCompatAccuracy}{63.33}

\newcommand{\SameTypeKFourOvos}{53.50}
\newcommand{\SameTypeKFourOst}{64.06}
\newcommand{\SameTypeKFourUcs}{51.70}
\newcommand{\SameTypeKFourReadout}{62.29}

\newcommand{\CrossTypeKFourOvos}{52.07}
\newcommand{\CrossTypeKFourOst}{64.34}
\newcommand{\CrossTypeKFourUcs}{51.77}
\newcommand{\CrossTypeKFourReadout}{66.88}

\newcommand{\StateWriteOvosUniform}{52.18}

\newcommand{\StateWriteOst}{63.42}
\newcommand{\StateWriteUcs}{54.41}
\newcommand{\StateWriteReadout}{58.13}

\newcommand{\NoTypeOst}{64.45}
\newcommand{\NoTypeReadout}{66.67}
\newcommand{\UnbalancedOst}{64.34}

\newcommand{\UnbalancedY}{28.75}
\newcommand{\UnbalancedN}{97.08}
\newcommand{\JointChoiceOst}{64.57}
\newcommand{\JointChoiceReadout}{63.96}
\newcommand{\UnnormalizedStateMassOst}{64.54}
\newcommand{\UnnormalizedStateMassReadout}{67.92}

\newcommand{\ProbeKOneRows}{1{,}920}
\newcommand{\ProbeKOneStates}{960}
\newcommand{\FinalProbeRows}{7{,}680}
\newcommand{\FinalProbeStates}{960}

\newcommand{\OstWithheldScenes}{243}

\newcommand{\method}{ChronoReplay\xspace}
\newcommand{\stitle}[1]{\vspace{1mm} \noindent {\bf #1}}
\newcommand{\chOne}{\ding{182}}
\newcommand{\chTwo}{\ding{183}}

\title{ChronoReplay: One Answer Interface Is Not Enough for Streaming Spatial Intelligence}

\author{Yihao~Hu\textsuperscript{1,\ensuremath{\ast}},
Yitong~Lin\textsuperscript{2,\ensuremath{\ast}},
Pan~Wang\textsuperscript{3,\ensuremath{\ast}},
Zhihao~Wen\textsuperscript{4,\ensuremath{\dagger}},
Hang~Wang\textsuperscript{1}, Yiwen~Zhao\textsuperscript{5},
Deyuan~Liu\textsuperscript{1}, Yutong~Li\textsuperscript{6},
Mingda~WAN\textsuperscript{1}, Yuchen~Mao\textsuperscript{7},
Tao~Lin\textsuperscript{1,\ensuremath{\dagger}}}
\affiliations{%
\textsuperscript{1}Westlake University,\quad
\textsuperscript{2}ZJU,\quad
\textsuperscript{3}USTC,\quad
\textsuperscript{4}SMU,\quad
\textsuperscript{5}FDU,\quad
\textsuperscript{6}ECNU,\quad
\textsuperscript{7}SJTU\\[4pt]
\textsuperscript{\ensuremath{\ast}}Equal contribution,\quad
\textsuperscript{\ensuremath{\dagger}}Corresponding author}

\begin{document}

\raggedbottom
\setlength{\parskip}{6pt}
\setlength{\textfloatsep}{6pt}
\setlength{\dbltextfloatsep}{6pt}
\setlength{\floatsep}{6pt}
\setlength{\dblfloatsep}{6pt}
\setlength{\intextsep}{6pt}

\maketitle

\begin{abstract}
Training a spatial VLM to answer which side an object lies on does not require
it to preserve the metric displacement that determines that answer, while
supervision for a distance estimate may leave categorical relations largely
unconstrained. When post-training emphasizes one readout, improvements can
therefore concentrate on that interface while transfer to others remains weak
or even degrades; we call this observed specialization \emph{interface lock-in}.
\method{} addresses it by replaying heterogeneous spatial supervision as
chronologically ordered episodes, retaining source-native targets, and
balancing supervision across families. Under the same observation access, a
single \method{} checkpoint improves OVO-S-Bench, OST-Bench, and UCS-Bench
together, including a \CsaOstTrainGain{}-point gain on OST-Bench, and achieves
the highest three-benchmark mean among the supervision settings compared here.
Same-data controls further separate the ingredients of this transfer: answer
serialization redistributes gains across OST and UCS, family-balanced learning
improves over ordinary joint SFT and natural turn weighting, and chronological
replay exceeds a shifted-order control by about \OrderGain{} OST points across
two paired seeds. Adaptive frame selection then changes evidence access while
keeping the trained model fixed, providing an additional inference-time gain
at lower visual cost. \textbf{Streaming spatial performance depends on what
distinctions supervision preserves, how heterogeneous histories are organized
during learning, and what causal evidence inference makes available at test time.}
\par\smallskip
{\small
\textbf{Code:} \url{https://github.com/yihaohu0118/ChronoReplay}\\[2pt]
\textbf{Project Homepage:} \url{https://yihaohu0118.github.io/ChronoReplay/}\\[2pt]
\textbf{Model and Training Annotations:}
\href{https://huggingface.co/collections/mis2/chronoreplay-6adc88dc1e468c2d2ae9bb44}{Hugging Face collection}\par}
\end{abstract}

\input{sections/01_introduction}
\input{sections/02_related}
\input{sections/03_problem_formulation}
\input{sections/04_method}
\input{sections/05_experiments}
\input{sections/06_conclusion}

\clearpage
\subsection*{AI use statement}

Generative AI was used in two roles.  AI assistants supported literature
search, research brainstorming, manuscript revision, code editing, debugging,
and data-pipeline implementation.  In upstream data preparation, multimodal
teachers were also used for bounded tasks such as proposing linguistic forms
or assessing whether source evidence was visible and answerable.  These
teachers did not define spatial ground truth, and protected evaluation labels,
predictions, and scores were unavailable to data construction.  Every target
used by \method{} is bound to a source annotation, deterministic geometry,
simulator state, or an audited public label; reported results are produced by
frozen benchmark scorers.  The authors designed the research, checked
AI-assisted code, citations, claims, and prose against source papers and
versioned experimental records, and take responsibility for the complete work.

\subsection*{Ethics statement}

This work uses existing research datasets and public benchmarks; it recruits
no human participants and collects no new video or indoor scans.  Source media
may contain identifiable people, private-like interiors, or contextual
information.  Any artifact release therefore follows the corresponding source
licenses and does not redistribute benchmark media.  Canonical scene
isolation prevents reported evaluation scenes and labels from entering
training.  Streaming spatial models may nevertheless inherit demographic,
geographic, and environmental biases from their source data, and their ability
to retain people, objects, and locations over time could enable surveillance
or unsafe embodied decisions.  We have not audited these high-stakes uses and
do not claim suitability for surveillance, autonomous navigation, or other
safety-critical deployment.

\subsection*{Reproducibility statement}

Appendices~\ref{app:protocol}--\ref{app:experimental-details} specify the
evaluation contracts, replay data, scene isolation, chronological
construction, optimization, and controlled comparisons.
Appendix~\ref{app:temporal-support} separately specifies the fixed-parameter
observation-access intervention, and Appendix~\ref{app:reproducibility} records
the executable configurations, generation parameters, scorer contracts, and
source manifests.  Every retained trainable comparison starts independently
from the same pinned public Qwen3-VL-8B-Instruct revision; no intermediate
checkpoint continuation or model fusion is used. Reported OVO-S, OST, and UCS
scores bind model, data, protocol, prediction, and scorer identities, and all
reported values are defined once as macros in \path{main.tex} and traced to
versioned result records.  Code, configurations, frozen training-selection
manifests, and scoring scripts are available at
\href{https://github.com/yihaohu0118/ChronoReplay}{the code repository};
\href{https://yihaohu0118.github.io/ChronoReplay/}{the project homepage}
provides the paper and an overview of the released resources.
Model weights and converted training annotations are available through a
single \href{https://huggingface.co/collections/mis2/chronoreplay-6adc88dc1e468c2d2ae9bb44}{Hugging Face collection}.
The data release preserves the frozen sample order, native targets, training
weights, and supervision masks; raw media and saved predictions are not bundled.
Releasing the accepted records makes reproducing
the reported training mixture independent of re-querying any external AI
service.

\clearpage
\bibliography{references}
\bibliographystyle{iclr2027_conference}

\clearpage
\appendix

\input{sections/09_theory}
\input{sections/08_appendix}

\end{document}

%% file: math_commands.tex
\usepackage{amsmath,amsfonts,bm}

\def\eqref#1{equation~\ref{#1}}

\def\1{\bm{1}}

\DeclareMathAlphabet{\mathsfit}{\encodingdefault}{\sfdefault}{m}{sl}
\SetMathAlphabet{\mathsfit}{bold}{\encodingdefault}{\sfdefault}{bx}{n}



%% file: sections/01_introduction.tex
\section{Introduction}

A home robot follows you through your apartment and is later asked:
\emph{Is the mirror on my left or right?  How far is the sofa from here?
When did we pass the plant?}  One walk, three questions, three answer forms:
an option, a number in meters, a moment in time.  Answering them from a video
stream means associating objects across views, accumulating the observer's
motion, and recovering relations in changing reference frames
\citep{yang2025thinking,plizzari2024spatial,yang2025mmsi,wang2026spatial4d,thompson2025rem}, using only what was observed before the
question; streaming spatial
benchmarks now pose all three over one causal history.

Recent progress comes from two directions.  Architectural work retains the
stream with visual memories, bounded caches, fast weights, maps, and scene
graphs \citep{song2024moviechat,shen2025longvu,liu2026spatialttt,gu2026spacemind}; data work
post-trains VLMs on spatial instruction data
\citep{wu2025spatialmllm,xu2025multispatialmllm,yang2025cambrians}.  Both
are typically \emph{single-interface}: a recipe supervises one answer format
and reports the benchmark that shares it; even recipes that mix qualitative
and metric questions \citep{chen2024spatialvlm,cheng2024spatialrgpt} train on
single images, so whether one model can serve categorical, metric, and
temporal queries over the same causal history is untested.
Building on this, we show that changing the supervision can achieve this
without a new memory module or new annotation.

Post-training on individual supervision families exposes two challenges.
On the one hand, \chOne{} \emph{training can improve its dominant
readout while transferring unevenly to other native readouts}. Starting from
Qwen3-VL-8B, the categorical-only control reaches
\CategoricalControlOvos{} on OVO-S but lowers OST from \BaseOst{} to
\CategoricalControlOst{}; trajectory-only training leaves OST nearly flat at
\TrajectoryControlOst{}. We use \emph{interface lock-in} to name this
empirical specialization across spatial readouts. The control motivates
joint training, while differences in source and task coverage limit
attribution to answer format alone (Appendix~\ref{app:single-family-control}).
On the other hand,
\chTwo{} \emph{a metric or temporal
target is valid only under the history that produced it}.  An observer that
moves forward after facing north and one that does so after facing east share
the same recent observations and the same local displacement, yet a later
question about a previously seen object has different answers.  Supervision
that detaches a target from its valid prefix can therefore be ambiguous;
it is crucial to keep each target
attached to the chronological prefix under which it holds.

\begin{figure}[t]
  \centering
  \vspace{-10pt}
  \includegraphics[width=\textwidth]{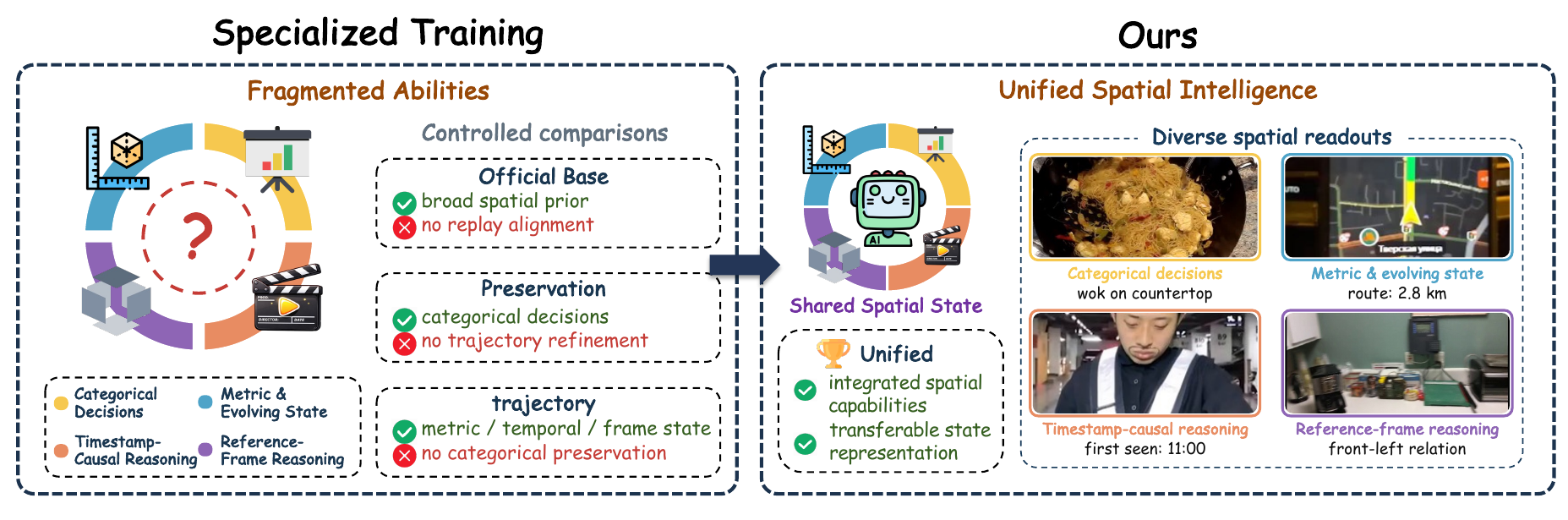}
  \vspace{-10pt}
  \caption{\textbf{Why reusable spatial state.} Categorical preservation and
  trajectory supervision constrain complementary spatial distinctions.
  Causal Replay combines them in shared parameters. The illustrated readouts
  span categorical, metric, temporal, and reference-frame queries.}
  \label{fig:spatial-teaser}
\end{figure}

In this paper, we propose \method{}, which consists of two components.
\textbf{First}, to address \chOne{}, \emph{family-balanced learning} keeps
every target in its \emph{native format} (option semantics, units and
precision, or timestamp) instead of collapsing it into one canonical answer
type, and balances the two families in one objective to regulate their
relative influence on the shared model. \textbf{Second}, to address \chTwo{},
\emph{Causal Replay construction} replays both supervision families as
chronologically ordered episodes in which every target is attached to the
causal prefix under which it is valid, so an accumulated displacement or a
retrieved timestamp is learned together with the origin and motion that give
it meaning (\emph{chronological alignment}). The main contributions are:
\begin{itemize}[topsep=-1pt,itemsep=1.5pt,parsep=0pt,partopsep=0pt]
\item We document an empirical \emph{interface lock-in} pattern (\chOne{}):
single-family training transfers unevenly across spatial readouts.
\item To address \chOne{} and \chTwo{}, we propose \method{}: family-balanced
learning trains all interfaces in their native formats under one objective,
and Causal Replay construction keeps every target attached to its
chronological prefix. \method{} gains +\CsaOstTrainGain{} OST-Bench,
+\CsaUcsTrainGain{} UCS-Bench, and +\CsaOvosTrainGain{} OVO-S-Bench points
under identical observation access, with OVO-S numerically comparable to a
published 235B reference.
\item Matched interventions show a repeatable chronological-order benefit
and a task-dependent answer-format tradeoff; adaptive frame selection at
inference adds a separate gain at lower visual cost.
\end{itemize}

%% file: sections/02_related.tex
\section{Related Work}
\label{sec:related}

\stitle{Streaming spatial intelligence.}
Streaming and long-video models process observations causally and keep
history in compressed visual memories, hierarchical summaries, recurrent
memory bridges, active response mechanisms, or bounded key--value caches
\citep{chen2024videollmonline,ren2024timechat,qian2024videostreaming,zhang2024flashvstream,he2024malmm,wang2025videollamb,shu2025videoxl,zhang2024omnilive,huang2025online,qian2025dispider,chen2025livecc,xiong2025streamchat,yang2025streammem,zeng2025streamforest}; StreamingBench and OVO-Bench formalize timestamped
evaluation under causal access \citep{lin2024streamingbench,niu2025ovobench},
and SimpleStream characterizes the trade-off
between recent perception and historical retention \citep{shen2026simple}.
Related memory designs for visual agents include L$^{2}$-VMAS, which separates
perception and thinking into latent memories \citep{yu2026duallatent}, and
AtlasVA, which combines spatial heatmaps, visual exemplars, and symbolic
skills to retain visually grounded experience \citep{wang2026atlasva}.
Spatial benchmarks further require history to stay coherent under object and
observer motion and changes of reference frame, from 3D question answering,
embodied spatial tests, and long-form egocentric video
\citep{azuma2022scanqa,ma2023sqa3d,du2024embspatial,wang2024spatialeval,mangalam2023egoschema} to multi-view, 4D, and mental-model probes
\citep{li2025stibench,wang2025mindcube}; OVO-S-Bench, OST-Bench,
and UCS-Bench extend them to causal streams with categorical, continuous,
and timestamp-dependent queries.  In this work, we hold
observation access fixed and treat it as an orthogonal factor
(Section~\ref{sec:access-decomposition}), and we use these three benchmarks
because their answer formats differ; the recipe adds no memory mechanism, and
our evaluation covers indoor scenes, so other environments remain untested.

\stitle{Learning reusable spatial representations.}
Spatial MLLMs acquire geometry through geometry-aware visual features,
multi-view or 3D-reconstruction inputs, and large-scale spatial instruction
or reinforcement data
\citep{cai2025spatialbot,liu2025coarse,zheng2025video3dllm,qi2025gpt4scene,fan2026vlm3r,ouyang2025spacer,wu2025vilasr}, and persistent spatial state has been
represented through 3D scene tokens and embodied generalist agents
\citep{hong2023threedllm,huang2024leo}, adaptive fast weights, voxelized
allocentric maps, textual scene descriptions, evolving scene graphs, and
action-conditioned belief histories
\citep{hua2026unleashing,liao2026spamem}.  Some mix qualitative and metric spatial
questions on single images,
and none that we know of asks whether categorical, metric, and temporal
formats interfere when one checkpoint is trained on all of them over one
causal stream.  This is a multi-task question: which
distinctions does a shared representation keep when objectives compete for
the same parameters
\citep{caruana1997multitask,sener2018mtl,standley2020tasks,zhang2022mtlsurvey},
how should their losses be balanced, and what does fine-tuning on
one distribution erase from another
\citep{kirkpatrick2017ewc,luo2025forgetting,zhai2023mllmforgetting}?
Instruction-tuning mixtures raise the same question at the level of data
\citep{wei2022flan,longpre2023flan,liu2023llava}, and the answer format
itself changes what a model appears to know.  In sequential state tracking,
PAR aligns recurrent predictions with intermediate states to improve length
generalization \citep{wang2026par}.  In this work, we examine the chronology of
training episodes and the format and weighting of their targets, and measure the
outcome through native benchmarks.

%% file: sections/03_problem_formulation.tex
\section{Problem Formulation}
\label{sec:formulation}

At time $t$ a model receives the causal observation history
$h_t=(o_1,\ldots,o_t)$.  A spatial query $q$ arrives with a native answer
interface $\kappa\in\mathcal K$ (a categorical option set, a metric value with
units, or a timestamp), and the VLM predicts
\begin{equation}
z_t=f_\theta(h_t),
\qquad
\widehat y_t^{(\kappa)}
=
g_\theta(z_t,q,\kappa),
\label{eq:streaming-prediction}
\end{equation}
against the task-defined target $y_t^{(\kappa)}=A_\kappa(h_t,q)\in\mathcal Y_\kappa$.
We call the representation $z_t$ \emph{reusable} over a family of
query--interface pairs when every native target in that family can be read
out from the same $z_t$.  Reusability is a functional requirement, not a
structural one: the state may stay distributed across the VLM, but histories
that lead to different answers must remain distinguishable.

Two properties of spatial histories make this demanding.  First,
interfaces refine one another.  A categorical target separates spatial regions, whereas a
metric target for the same quantity resolves variation within a region, so a
representation that keeps only the coarse partition cannot support the finer
readout.  Second, the state depends on the origin of the history.  An observer
that moves forward after facing north and one that moves forward after facing
east share the same recent observations and the same local displacement, yet a
later question about a previously seen object requires different answers.
Histories with identical recent suffixes but different origins must therefore
map to different states.  Our goal is to learn one $f_\theta$ with both
properties from heterogeneous native supervision;
Section~\ref{sec:method} builds training episodes that expose both.

%% file: sections/04_method.tex
\section{\method}
\label{sec:method}

\subsection{Overview}

\method{} post-trains a VLM so that one shared representation supports native
spatial readouts while preserving the history each readout depends on.
Figure~\ref{fig:chronoreplay-overview} summarizes two stages.  \emph{Causal
Replay construction} forms the replay pool: every target is kept in its
native format and every episode is ordered chronologically.  \emph{Family-balanced learning} normalizes the
single-turn preservation and metric--temporal trajectory risks, then controls
their relative contribution to the shared parameters.
Algorithm~\ref{alg:chronoreplay} in Appendix~\ref{app:replay-data} lists both
stages.

\begin{figure}[!htbp]
  \centering
  \vspace{-4pt}
  \includegraphics[width=1.00\textwidth]{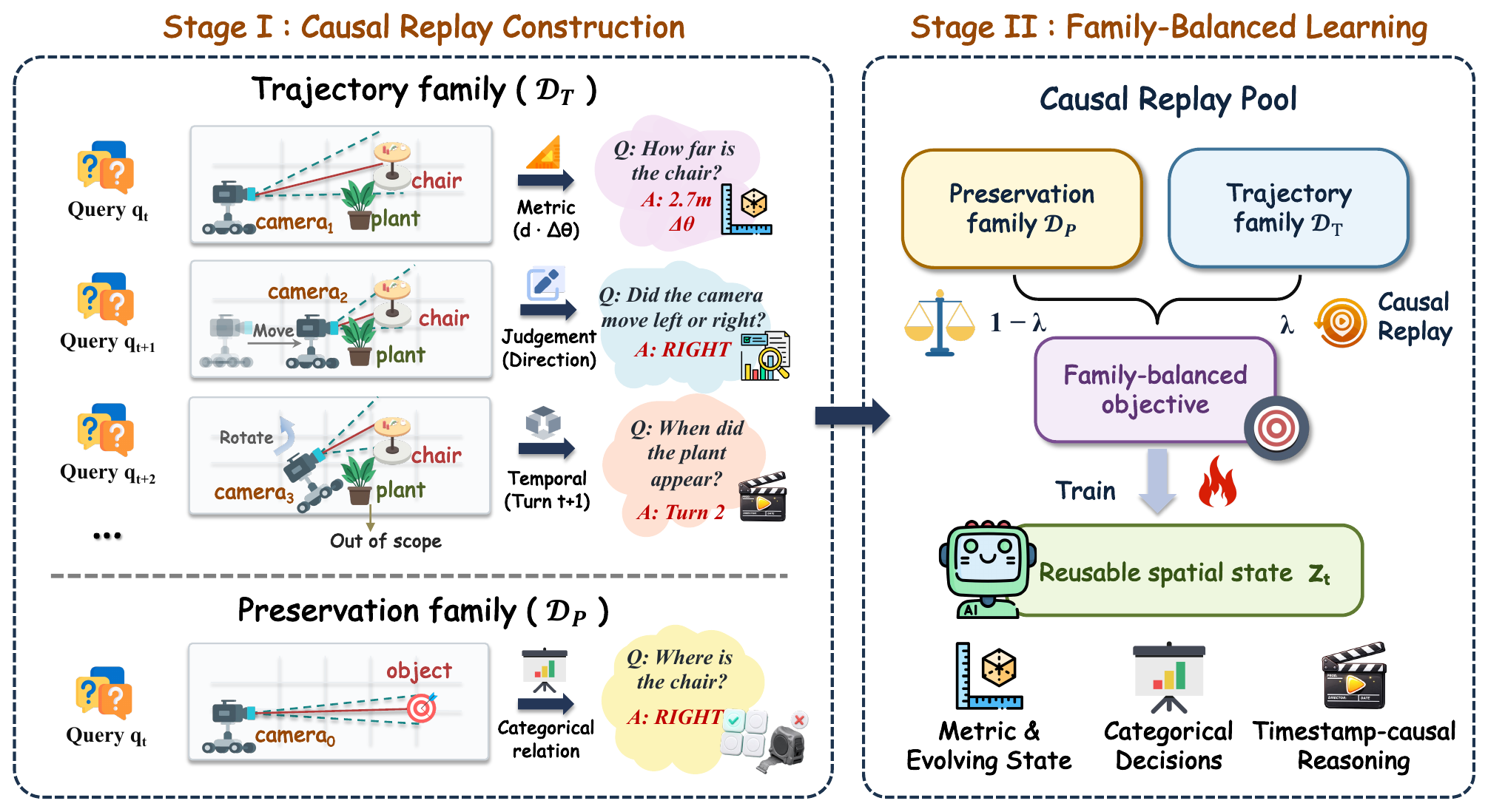}
  \vspace{-10pt}
  \caption{\textbf{Overview of \method{}.} Stage I (Causal Replay
  construction) builds ordered trajectory supervision $\mathcal D_T$ and
  single-turn preservation examples $\mathcal D_P$. Stage II (family-balanced
  learning) combines their normalized native-format risks with weights
  $1-\lambda$ and $\lambda$.  \emph{Out of scope} marks observations beyond a
  query's causal frontier.}
  \label{fig:chronoreplay-overview}
  \vspace{-10pt}
\end{figure}

\subsection{Causal Replay Construction}
\label{sec:causal-replay-construction}

\subsubsection{Native-Format Targets}
\label{sec:native-refinement}

\stitle{Motivation.}
\emph{How should heterogeneous answer spaces be combined?}
Considering \chOne{}, we keep each target in the value space of its source
interface, avoiding an additional answer-space transformation when combining
categorical options, metric quantities with units and precision, and temporal
answers. This is a design choice for retaining source task semantics, rather
than a claim that native encoding is uniformly superior. The format through
which a model answers can change what it appears to know
\citep{wang2024myanswer} and the decoder can emit metric values directly
\citep{vacareanu2024regressor}.

\stitle{Native readouts.}
Let $F_\theta(C)$ denote the VLM representation of a causal prefix $C$
(Section~\ref{sec:formulation}).  For interface $\kappa$, let $Y^\kappa$
denote its native target and $G_\kappa$ the readout applied to
$Z=F_\theta(C)$.  Its population risk is
\begin{equation}
\mathcal R_\kappa(F_\theta,G_\kappa)
=
\mathbb E
\left[
-\log G_\kappa
\left(
Y^\kappa\mid Z,Q
\right)
\right].
\label{eq:native-interface-risk}
\end{equation}
For a fixed representation and a sufficiently expressive readout, the optimal
risk equals the conditional entropy $H(Y^\kappa\mid Z,Q)$, so a low risk
requires $Z$ to retain the information needed to recover the native target.
Proposition~\ref{prop:sufficiency} in Appendix~\ref{app:theory} makes this
precise. Corollary~\ref{cor:canonical} shows that lossless re-encoding
preserves this information requirement. Our native/canonical control changes
output serialization without losing target content and finds crossed
OST/UCS transfer effects.

\stitle{Retained value spaces.}
Categorical targets keep their candidate structure; metric targets keep
magnitude, sign, units, and annotation precision; temporal targets keep
timestamps, event indices, and transition-relative quantities;
reference-frame information stays coupled to the query.  All targets are
produced by the original autoregressive decoder.

\subsubsection{Chronologically Aligned State Trajectories}
\label{sec:origin-aligned-trajectories}

\emph{What does a metric or temporal target mean without its history?}
With each target's format fixed, we now fix its history.  Considering
\chTwo{}, the same local displacement yields different world states from
different origins (Section~\ref{sec:formulation}), so a target detached from
its prefix is ambiguous. We use \emph{temporal-origin alignment} to denote a
data-construction rule: supervision from the interfaces present in a source
episode is placed on that episode's video-time axis and replayed in its
original order, with each target attached to its legal causal prefix.
Unlike flattening multi-round annotations into independent question--answer
pairs, chronological replay preserves the source order and legal conditioning
history used by the streaming evaluation protocol.

\stitle{Ordered episodes.}
Let $c_{ij}$ denote the legal causal prefix at turn $j$ of episode $i$: all
observations available at that turn together with the preceding interaction
history.  We represent an episode as
\begin{equation}
\mathcal E_i
=
\left(
c_{ij},q_{ij},y_{ij},\kappa_{ij}
\right)_{j=1}^{m_i},
\qquad
c_{ij}\preceq c_{i,j+1},
\label{eq:replay-episode}
\end{equation}
where $\preceq$ denotes causal prefix inclusion, $q_{ij}$ is a spatial query,
$\kappa_{ij}$ identifies its native interface, and
$y_{ij}=A_{\kappa_{ij}}(c_{ij},q_{ij})$ is the target valid at that frontier.
Single-turn preservation examples correspond to $m_i=1$, whereas multi-round
metric--temporal examples form longer episodes.  Because the prefixes are
nested, one episode yields $m_i$ supervised states that share an origin and
grow in history length, so the model is trained on accumulation itself, on
how a displacement or an event order changes from turn to turn, rather than
on isolated snapshots.

\stitle{Implicit state trajectory.}
Applying the VLM to successive prefixes induces
\begin{equation}
z_{ij}=F_\theta(c_{ij}),
\qquad
\mathcal Z_i
=
\left(z_{i1},\ldots,z_{im_i}\right).
\label{eq:state-trajectory}
\end{equation}
The trajectory stays implicit within the VLM, and its first prefix defines
the episode origin. No map, memory module, or auxiliary state loss is added.
The origin-shift control in Section~\ref{sec:controlled-analysis} changes the
ordering and preceding context while retaining each turn block and its local
supervision. It tests the contribution of this construction rule as a whole;
it does not identify a particular latent origin representation. Proposition~
\ref{prop:origin} in Appendix~\ref{app:theory} characterizes the risk when
the changed context makes the target ambiguous.

\subsection{Family-Balanced Learning}
\label{sec:family-balanced-learning}

\stitle{Motivation.}
\emph{What if the two families differ in size?}  With targets and histories
fixed, the remaining choice is how much each family may move the shared
parameters. The preservation family $\mathcal D_P$ holds \PreservationRows{}
single-turn examples (30,635 categorical and 2,545 numerical), whereas the trajectory family
$\mathcal D_T$ holds \TrajectoryRows{} ordered episodes (\TrajectoryTurns{}
supervised metric, counting, judgement, and temporal turns).  Raw row counts
could let one family dominate the shared parameters and reproduce \chOne{}
inside joint training.  Therefore, we normalize risk within each family and
set their ratio explicitly rather than learning it
\citep{kendall2018uncertainty,chen2018gradnorm}.

\stitle{Family risks.}
Let $\ell_r^P$ be the native loss for preservation example $r$, with
source-defined weight $w_r$ (Appendix~\ref{app:replay-data}). Let $\ell_{ij}^T$ be the native loss for verified
trajectory transition $(i,j)$, and define
$\mathcal T=\{(i,j):1\leq i\leq|\mathcal D_T|,\ 1\leq j\leq m_i\}$.  The two
family risks are
\begin{equation}
\mathcal L_P
=
\frac{
\sum_{r\in\mathcal D_P}w_r\ell_r^P
}{
\sum_{r\in\mathcal D_P}w_r
},
\qquad
\mathcal L_T
=
\frac{
\sum_{(i,j)\in\mathcal T}\ell_{ij}^T
}{
|\mathcal T|
}.
\label{eq:family-risks}
\end{equation}

\stitle{Balanced objective.}
\method{} combines the risks as
\begin{equation}
\mathcal L_{\mathrm{replay}}
=
(1-\lambda)\,\mathcal L_P
+
\lambda\,\mathcal L_T,
\qquad
\lambda=\frac{\alpha}{1+\alpha},
\label{eq:causal-replay-objective}
\end{equation}
where $\alpha$ is trajectory loss mass relative to preservation. We use
$\alpha=\TrajectoryMass{}$ in the main experiments: trajectory receives $1/3$
of total mass; preservation receives $2/3$, split equally between categorical
and numerical targets. Appendix~\ref{app:experimental-details} specifies token
normalization and serialized weighting; Proposition~\ref{prop:balance} states
their invariances. Table~\ref{tab:joint-objectives} compares ordinary token SFT
and natural turn weighting on the same replay pool.

Answer-length normalization assigns each supervised turn an average token loss,
preventing long responses from dominating solely through token count. Within
the trajectory family, turns receive equal weight, so an episode contributes
in proportion to its supervised turns. Family weighting controls preservation
and trajectory loss mass separately. This distinguishes response length,
supervision frequency, and family balance within the training objective.

%% file: sections/05_experiments.tex
\section{Experiments}
\label{sec:experiments}

In this section, we evaluate \method{} from multiple perspectives and
organize the analysis around four research questions.
\textbf{RQ1:} Does \method{} transfer across categorical,
continuous, and timestamp-causal benchmarks under fixed observation access
(Section~\ref{sec:main-results})?
\textbf{RQ2:} How do single-family and joint training transfer across spatial
readouts
(Section~\ref{sec:controlled-analysis})?
\textbf{RQ3:} Which components carry the effect: native formats,
chronological order, or family mass (Section~\ref{sec:controlled-analysis})?
\textbf{RQ4:} How much of the OVO-S system score is observation access rather
than training (Section~\ref{sec:access-decomposition})?

\subsection{Experimental Setup}
\label{sec:experimental-setup}

We evaluate on three benchmarks: OVO-S-Bench tests categorical reasoning
on egocentric perception, object permanence, and allocentric mapping
\citep{li2026ovosbench}; OST-Bench scores agent state, visible information,
temporal localization, counting, and agent--object relations over multi-round
exploration \citep{lin2025ostbench}; and UCS-Bench has 8,114 timestamp-causal
questions on trajectories, pose, proximity, object identity, and quantity
\citep{wang2026keep}.  All checkpoints start from the same official
Qwen3-VL-8B-Instruct revision \citep{qwen3vl2025}.  \method{} trains the
language model and visual merger (visual encoder frozen) on the
\PreservationRows{}-sequence preservation family and the
\TrajectoryRows{}-episode trajectory family with $\alpha=\TrajectoryMass{}$,
and canonical scene identities keep every test scene out of training.
The single-family diagnostic compares \textsc{Categorical-only} and
\textsc{Trajectory-only} variants with full \method{}
(Appendix~\ref{app:single-family-control}).
Its categorical-only arm is distinct from the main preservation pool, which
also contains numerical targets.
\textsc{Shifted} cyclically rotates the turn blocks of every episode while
keeping their content, controlling for chronological alignment;
\textsc{Canonical} rewrites every target as a typed \texttt{value} field,
controlling for native formats (the \method{} default is \textsc{Native});
and Uniform-81 matches the average frame count of adaptive access, an
equal-cost control for observation access.  Model-only comparisons hold causal
observation access fixed (OVO-S Uniform-128, the complete scene-disjoint OST
protocol, all UCS questions), while adaptive access
(Appendix~\ref{app:temporal-support}) is an inference-time policy applied
separately to the same checkpoint.  Published models are reference points
under their own protocols, and mechanism claims rely on matched Qwen3-VL-8B
comparisons.  Appendix~\ref{app:experimental-details} lists optimization
settings, matched controls, and a four-GPU replication.

\subsection{Cross-Interface Transfer (RQ1)}
\label{sec:main-results}

\begin{table}[!htb]
\caption{\textbf{Complete task-family results on OVO-S-Bench} (\%), twelve
official task families. Our rows report accuracy within each family; Overall
is their unweighted mean. Published reference scores are quoted from the
benchmark paper. Uniform-128 is the matched training comparison; the final
row applies adaptive access to the same trained model.}
\label{tab:ovos-main}
\centering
\scriptsize
\setlength{\tabcolsep}{1.8pt}
\renewcommand{\arraystretch}{1.08}
\resizebox{\textwidth}{!}{%
\begin{tabular}{l*{14}{c}}
\toprule
& \multicolumn{3}{c}{Level 1}
& \multicolumn{3}{c}{Level 2}
& \multicolumn{3}{c}{Level 3}
& \multicolumn{3}{c}{Level 4}
& Overall & Frames\\
\cmidrule(lr){2-4}
\cmidrule(lr){5-7}
\cmidrule(lr){8-10}
\cmidrule(lr){11-13}
Model
& 1.1 & 1.2 & 1.3
& 2.1 & 2.2 & 2.3
& 3.1 & 3.2 & 3.3
& 4.1 & 4.2 & 4.3
& & \\
\midrule
\rowcolor{GroupGray}
\multicolumn{15}{l}{\textit{Representative official prefix-access models}}\\
InternVL-3.5-8B
& 47.86 & 54.60 & 35.31
& 40.95 & 40.31 & 56.22
& 54.00 & 26.98 & 60.61
& 34.65 & 55.00 & 28.30
& 44.57 & 128\\
Qwen3-VL-4B
& 32.91 & 53.19 & 43.65
& 47.26 & 34.25 & 63.04
& 57.33 & 38.10 & 68.18
& 25.74 & 70.00 & 28.30
& 46.83 & 128\\
Qwen3-VL-32B
& 48.83 & 57.07 & 44.51
& 50.71 & 34.58 & 70.43
& 60.67 & 26.98 & 68.18
& 32.67 & 68.33 & 22.64
& 48.80 & 128\\
InternVL-3.5-38B
& 58.64 & 55.56 & 49.74
& 50.30 & 48.00 & 64.91
& 61.33 & 22.22 & 53.03
& 38.61 & 68.33 & 18.87
& 49.13 & 128\\
InternVL-3.5-241B-A28B
& 61.30 & 59.36 & 45.98
& 48.75 & 49.86 & 68.54
& 53.33 & 33.33 & 68.18
& 42.57 & 65.83 & 13.21
& 50.85 & 128\\
GPT-5.4
& 52.08 & 63.70 & 48.13
& 46.25 & 57.92 & 68.61
& 61.33 & 36.51 & 54.55
& 32.67 & 62.50 & 26.42
& 50.89 & 128\\
Qwen3.5-27B
& 46.50 & 58.13 & 49.88
& 55.24 & 41.53 & 68.67
& 54.00 & 36.51 & 66.67
& 35.64 & 71.67 & 35.85
& 51.69 & 128\\
Qwen3.5-397B-A17B
& 46.71 & 53.32 & 48.68
& 55.65 & 41.22 & 69.30
& 57.33 & 38.10 & 78.79
& 34.65 & 75.00 & 26.42
& 52.10 & 128\\
Qwen3-VL-235B-A22B
& 50.36 & 56.09 & 50.91
& 49.76 & 45.61 & 70.17
& 63.33 & 46.03 & 74.24
& 45.54 & 70.83 & 20.75
& 53.64 & 128\\
Gemini-3.1-Pro
& 58.69 & 71.18 & 55.89
& 57.92 & 59.97 & 74.23
& 67.33 & 41.27 & 59.09
& 60.40 & 81.67 & 22.64
& 59.19 & 128\\
\midrule
\rowcolor{GroupGray}
\multicolumn{15}{l}{\textit{Our Qwen3-VL-8B evaluation under matched causal access}}\\
Official Qwen3-VL-8B (Uniform-128)
& 57.49 & 60.31 & 44.40
& 54.55 & 47.02 & 70.07
& 59.33 & 26.98 & 62.12
& 19.85 & 61.11 & 24.53
& \BaseOvosUniform{} & 128\\
\rowcolor{MethodBlue}
\textbf{\method{} (Uniform-128)}
& \CsaOvosUniformFOneOne{} & \CsaOvosUniformFOneTwo{} & \CsaOvosUniformFOneThree{}
& \CsaOvosUniformFTwoOne{} & \CsaOvosUniformFTwoTwo{} & \CsaOvosUniformFTwoThree{}
& \CsaOvosUniformFThreeOne{} & \CsaOvosUniformFThreeTwo{} & \CsaOvosUniformFThreeThree{}
& \CsaOvosUniformFFourOne{} & \CsaOvosUniformFFourTwo{} & \CsaOvosUniformFFourThree{}
& \textbf{\CsaOvosUniform{}} & 128\\
\method{} + adaptive access
& \CsaOvosFOneOne{} & \CsaOvosFOneTwo{} & \CsaOvosFOneThree{}
& \CsaOvosFTwoOne{} & \CsaOvosFTwoTwo{} & \CsaOvosFTwoThree{}
& \CsaOvosFThreeOne{} & \CsaOvosFThreeTwo{} & \CsaOvosFThreeThree{}
& \CsaOvosFFourOne{} & \CsaOvosFFourTwo{} & \CsaOvosFFourThree{}
& \CsaOvosSupport{} & \MeanFrames{}\\
\bottomrule
\end{tabular}}
\end{table}

\stitle{Categorical streaming reasoning.}
\textbf{Under matched Uniform-128 access, \method{} raises OVO-S from
\BaseOvosUniform{} to \CsaOvosUniform{}.} The score is numerically comparable
to the published Qwen3-VL-235B-A22B reference. All twelve families improve;
L4 rises from \BaseOvosLfour{} to \CsaOvosUniformLfour{} but remains the
lowest-scoring level. The adaptive-access row is a system effect that
Section~\ref{sec:access-decomposition} separates from training.

\begin{table}[!htb]
\caption{\textbf{Fine-grained results on OST-Bench} (\%), official fifteen
subtypes. Published rows are reference points under their reported full-split
protocols; shaded rows use our matched scene-disjoint evaluation.
$\Delta$ is \method{} minus base. JUD., TEMP., CNT., and EST. denote
judgement, temporal localization, counting, and estimation.}
\label{tab:ost-main}
\centering
\scriptsize
\setlength{\tabcolsep}{1.65pt}
\renewcommand{\arraystretch}{1.10}
\resizebox{\textwidth}{!}{%
\begin{tabular}{lcccccccccccccccc}
\toprule
& \multirow{3}{*}{Avg}
& \multicolumn{4}{c}{Agent State}
& \multicolumn{5}{c}{Agent Visible Information}
& \multicolumn{6}{c}{Agent--Object Spatial Relationship}\\
\cmidrule(lr){3-6}
\cmidrule(lr){7-11}
\cmidrule(lr){12-17}
& &
\multicolumn{2}{c}{Position}
& \multicolumn{2}{c}{Orientation}
& \multicolumn{2}{c}{Existence}
& \multicolumn{1}{c}{Quantity}
& \multicolumn{1}{c}{Diversity}
& \multicolumn{1}{c}{Order}
& \multicolumn{3}{c}{Direction}
& \multicolumn{3}{c}{Distance}\\
\cmidrule(lr){3-4}
\cmidrule(lr){5-6}
\cmidrule(lr){7-8}
\cmidrule(lr){9-9}
\cmidrule(lr){10-10}
\cmidrule(lr){11-11}
\cmidrule(lr){12-14}
\cmidrule(lr){15-17}
Model
&
& JUD. & EST.
& JUD. & EST.
& JUD. & TEMP.
& CNT.
& JUD.
& JUD.
& JUD. & TEMP. & EST.
& JUD. & TEMP. & EST.\\
\midrule
\rowcolor{GroupGray}
\multicolumn{17}{l}{\textit{Representative published MLLMs}}\\
InternVL-2.5-8B
& 44.6
& 51.7 & 26.1
& 49.4 & 40.4
& 86.3 & 51.3
& 56.4
& 60.7
& 38.4
& 37.2 & 33.8 & 22.8
& 43.0 & 42.9 & 27.9\\
Qwen2.5-VL-72B
& 45.6
& 57.0 & 27.6
& 52.2 & 37.1
& 86.1 & 64.5
& 61.5
& 75.7
& 34.5
& 41.4 & 21.1 & 8.2
& 44.5 & 39.3 & 18.7\\
Qwen2.5-VL-32B
& 46.9
& 51.0 & 31.1
& 53.5 & 39.4
& 85.3 & 64.8
& 59.2
& 73.4
& 41.8
& 39.5 & 24.9 & 25.7
& 43.6 & 39.1 & 20.3\\
Claude-3.5-Sonnet
& 47.8
& 65.1 & 36.2
& 50.6 & 30.3
& 85.8 & 67.0
& 57.9
& 57.1
& 60.0
& 39.2 & 18.3 & 21.9
& 43.8 & 54.9 & 19.0\\
GPT-4o
& 48.7
& 55.6 & 20.5
& 45.6 & 33.6
& 90.6 & 75.6
& 59.8
& 78.2
& 59.6
& 46.1 & 19.5 & 21.4
& 43.1 & 50.6 & 20.4\\
Gemini-2.0-Flash
& 49.5
& 59.7 & 27.3
& 56.7 & 36.5
& 89.6 & 70.8
& 59.4
& 78.7
& 55.6
& 42.0 & 17.8 & 21.5
& 45.5 & 48.9 & 27.1\\
InternVL-2.5-38B
& 50.8
& 56.7 & 31.8
& 54.6 & 38.4
& 91.7 & 74.7
& 61.1
& 79.8
& 62.1
& 42.1 & 20.6 & 27.7
& 42.7 & 42.5 & 28.1\\
InternVL-2.5-78B
& 51.1
& 60.8 & 34.4
& 49.9 & 40.7
& 90.7 & 74.4
& 65.9
& 77.9
& 61.2
& 43.4 & 22.4 & 17.8
& 46.7 & 44.4 & 22.9\\
GPT-4.1
& 53.4
& 64.2 & 30.7
& 60.8 & 33.2
& 90.8 & 78.0
& 60.6
& 82.1
& 70.8
& 51.5 & 23.4 & 28.6
& 44.9 & 53.6 & 23.9\\
Gemini-2.0-Flash (Thinking)
& 54.2
& 57.4 & 36.3
& 61.1 & 33.4
& 88.1 & 74.8
& 61.9
& 63.6
& 73.1
& 50.9 & 51.7 & 23.3
& 51.1 & 56.8 & 22.7\\
\midrule
\rowcolor{GroupGray}
\multicolumn{17}{l}{\textit{Our scene-disjoint Qwen3-VL-8B evaluation}}\\
Official base
& \BaseOst{}
& \BaseOstPosJudge{}
& \BaseOstPosEst{}
& \BaseOstOriJudge{}
& \BaseOstOriEst{}
& \BaseOstExistJudge{}
& \BaseOstExistTemporal{}
& \BaseOstQuantity{}
& \BaseOstDiversity{}
& \BaseOstOrder{}
& \BaseOstDirJudge{}
& \BaseOstDirTemporal{}
& \BaseOstDirEst{}
& \BaseOstDistJudge{}
& \BaseOstDistTemporal{}
& \BaseOstDistEst{}\\
\rowcolor{MethodBlue}
\textbf{\method{}}
& \textbf{\CsaOst{}}
& \CsaOstPosJudge{}
& \CsaOstPosEst{}
& \CsaOstOriJudge{}
& \CsaOstOriEst{}
& \CsaOstExistJudge{}
& \CsaOstExistTemporal{}
& \CsaOstQuantity{}
& \CsaOstDiversity{}
& \CsaOstOrder{}
& \CsaOstDirJudge{}
& \CsaOstDirTemporal{}
& \CsaOstDirEst{}
& \CsaOstDistJudge{}
& \CsaOstDistTemporal{}
& \CsaOstDistEst{}\\
$\Delta$
& +\CsaOstTrainGain{}
& \DeltaOstPosJudge{}
& \DeltaOstPosEst{}
& \DeltaOstOriJudge{}
& \DeltaOstOriEst{}
& \DeltaOstExistJudge{}
& \DeltaOstExistTemporal{}
& \DeltaOstQuantity{}
& \DeltaOstDiversity{}
& \DeltaOstOrder{}
& \DeltaOstDirJudge{}
& \DeltaOstDirTemporal{}
& \DeltaOstDirEst{}
& \DeltaOstDistJudge{}
& \DeltaOstDistTemporal{}
& \DeltaOstDistEst{}\\
\bottomrule
\end{tabular}}
\end{table}

\stitle{Continuous state estimation.}
\textbf{On the scene-disjoint OST protocol, \method{} improves OST from
\BaseOst{} to \CsaOst{} (+\CsaOstTrainGain{} points).}
The $\Delta$ row of Table~\ref{tab:ost-main} is positive in all fifteen
subtypes. The
largest changes are Direction Estimation (\DeltaOstDirEst{}), Order
Judgement (\DeltaOstOrder{}), and Distance Temporal Localization
(\DeltaOstDistTemporal{}).  In short, gains are largest where an answer must be composed from
accumulated multi-turn state rather than the current view, showing that
\method{} substantially improves continuous state estimation over multi-round
exploration (RQ1).

\begin{table}[!htb]
\caption{\textbf{Complete UCS-Bench decomposition} over all 8,114
timestamp-causal questions (\%). Each group reports its aggregate (All) and
two subcategories; category names follow Appendix~\ref{app:protocol}.
Published rows are reference points under their reported evaluation
protocols; $\Delta$ is \method{} minus base.}
\label{tab:ucs-main}
\centering
\scriptsize
\setlength{\tabcolsep}{1.8pt}
\renewcommand{\arraystretch}{1.12}
\resizebox{\textwidth}{!}{%
\begin{tabular}{lcc*{13}{c}}
\toprule
& & & \multicolumn{3}{c}{Trajectory}
& \multicolumn{3}{c}{Position}
& \multicolumn{3}{c}{Proximity}
& \multicolumn{3}{c}{Category}
& Overall\\
\cmidrule(lr){4-6}
\cmidrule(lr){7-9}
\cmidrule(lr){10-12}
\cmidrule(lr){13-15}
Model & Size & Frames
& All & My & Obj.
& All & Ego & Rel.
& All & Dist. & Reach
& All & Cat. & Qty.
&\\
\midrule
\rowcolor{GroupGray}
\multicolumn{16}{l}{\textit{Representative official references}}\\
DeepSeek-V3 & -- & 50
& 42.0 & 41.0 & 44.8
& 33.2 & 35.4 & 31.6
& 37.7 & 45.9 & 31.2
& 32.3 & 36.8 & 30.3
& 35.6\\
GPT-5 & -- & 50
& 44.5 & 43.3 & 48.1
& 37.8 & 40.2 & 36.3
& 55.4 & 52.9 & 57.5
& 41.4 & 58.8 & 33.5
& 43.7\\
InternVideo2.5 & 8B & 64
& 44.2 & 43.9 & 45.2
& 39.4 & 42.4 & 37.5
& 52.7 & 49.1 & 55.5
& 42.2 & 53.0 & 37.4
& 43.8\\
Qwen3-VL-8B & 8B & 64
& 44.6 & 43.1 & 48.7
& 40.1 & 43.6 & 37.9
& 50.9 & 49.2 & 52.1
& 42.9 & 52.8 & 38.4
& 44.0\\
Doubao-Seed-1.6V & -- & 50
& 49.6 & 49.4 & 50.1
& 42.8 & 46.4 & 40.4
& 56.0 & 54.9 & 56.8
& 44.6 & 57.0 & 39.0
& 47.3\\
Qwen2.5-VL-32B & 32B & 64
& 50.3 & 49.4 & 53.0
& 40.9 & 43.1 & 39.5
& 55.4 & 50.4 & 59.3
& 46.0 & 56.4 & 41.4
& 47.3\\
Qwen3-VL-32B & 32B & 64
& 50.3 & 49.0 & 54.1
& 46.5 & 50.1 & 44.3
& 55.2 & 53.9 & 56.3
& 46.7 & 58.8 & 41.2
& 49.0\\
InternVL3-38B & 38B & 64
& 53.7 & 52.4 & 57.3
& 45.6 & 46.9 & 44.8
& 55.3 & 54.0 & 56.3
& 45.8 & 60.4 & 39.3
& 49.2\\
DirectMe (Qwen3-VL) & 8B & 1fps
& 54.1 & 53.0 & 57.2
& 47.6 & 50.2 & 46.0
& 58.0 & 55.7 & 59.8
& 46.4 & 53.7 & 43.1
& 50.5\\
Human & -- & --
& 91.7 & 90.9 & 94.1
& 87.6 & 84.0 & 89.8
& 96.8 & 95.0 & 98.2
& 90.0 & 94.3 & 88.1
& 91.0\\
\midrule
\rowcolor{GroupGray}
\multicolumn{16}{l}{\textit{Our Qwen3-VL-8B evaluation}}\\
Official Qwen3-VL-8B & 8B & 128
& \BaseUcsTrajectory{} & 47.49 & 51.49
& \BaseUcsEgo{} & 40.32 & 42.61
& \BaseUcsProximity{} & 46.12 & 52.14
& \BaseUcsObject{} & 53.50 & 39.60
& \BaseUcs{}\\
\rowcolor{MethodBlue}
\textbf{\method{}} & 8B & 128
& \CsaUcsTrajectory{} & \CsaUcsMyTrajectory{} & \CsaUcsObjectTrajectory{}
& \CsaUcsEgo{} & \CsaUcsEgoPosition{} & \CsaUcsRelativePose{}
& \CsaUcsProximity{} & \CsaUcsDistance{} & \CsaUcsReachability{}
& \CsaUcsObject{} & \CsaUcsObjectCategory{} & \CsaUcsQuantity{}
& \textbf{\CsaUcs{}}\\
$\Delta$ & &
& \DeltaUcsTrajectory{} & \DeltaUcsMyTrajectory{} & \DeltaUcsObjectTrajectory{}
& \DeltaUcsEgo{} & \DeltaUcsEgoPosition{} & \DeltaUcsRelativePose{}
& \DeltaUcsProximity{} & \DeltaUcsDistance{} & \DeltaUcsReachability{}
& \DeltaUcsObject{} & \DeltaUcsObjectCategory{} & \DeltaUcsQuantity{}
& +\CsaUcsTrainGain{}\\
\bottomrule
\end{tabular}}
\end{table}

\stitle{Timestamp-causal reasoning.}
\textbf{On UCS, \method{} improves from \BaseUcs{} to \CsaUcs{}
(+\CsaUcsTrainGain{} points), with positive transfer in all eight
subcategories.} The largest subcategory gains are on quantity
(\DeltaUcsQuantity{}), object
trajectory (\DeltaUcsObjectTrajectory{}), and observer trajectory
(\DeltaUcsMyTrajectory{}). These gains accompany categorical and metric
improvements under unchanged observation access (RQ1).

\subsection{Diagnostic Analysis (RQ2, RQ3)}
\label{sec:controlled-analysis}

\begin{table}[!htb]
\caption{\textbf{Diagnostic analysis of \method{}} (\%).
Panel (a): categorical-only, trajectory-only, and full \method{} on
Uniform-128 OVO-S and complete OST/UCS. \method{} uses the canonical model
from Tables~\ref{tab:ovos-main}--\ref{tab:ucs-main}; Mean averages the three benchmarks.
Panel (b): each alternative is compared with its corresponding \method{}
reference under the scope shown. Order averages two seed pairs;
the mass sweep uses adaptive OVO-S.}
\label{tab:controlled-analysis}
\centering
\small
\setlength{\tabcolsep}{2pt}
\renewcommand{\arraystretch}{1.10}
\resizebox{\textwidth}{!}{%
\begin{tabular}{lcccc@{\hspace{10pt}}lccc}
\toprule
\multicolumn{5}{c}{\textbf{(a) Single-family supervision}}
& \multicolumn{4}{c}{\textbf{(b) Mechanism controls}}\\
\cmidrule(lr){1-5}\cmidrule(lr){6-9}
Training & OVO-S & OST & UCS & Mean
& Control: alternative & Alt. & \method{} & Scope\\
\midrule
Official base & \BaseOvosUniform{} & \BaseOst{} & \BaseUcs{} & \BaseMean{}
& Format: \textsc{Canonical} & \CommonInterfaceMean{}
& \cellcolor{MethodBlue}\textbf{\NativeInterfaceMean{}} & OST/UCS\\
\textsc{Categorical-only} & \textbf{\CategoricalControlOvos{}}
& \CategoricalControlOst{} & \CategoricalControlUcs{} & \CategoricalControlMean{}
& Order: \textsc{Shifted} & \ShiftedTwoSeedMean{}
& \cellcolor{MethodBlue}\textbf{\ChronoTwoSeedMean{}} & OST\\
\textsc{Trajectory-only} & \TrajectoryControlOvos{}
& \TrajectoryControlOst{} & \TrajectoryControlUcs{} & \TrajectoryControlMean{}
& Mass: $\alpha=1.0$ & 60.92
& \cellcolor{MethodBlue}\textbf{61.20} & OVO-S/OST\\
\rowcolor{MethodBlue}
\textbf{\method{}} & \textbf{\CsaOvosUniform{}} & \textbf{\CsaOst{}}
& \textbf{\CsaUcs{}} & \textbf{\HeadlineMean{}}
& \cellcolor{white}Topology: 4 GPUs & \cellcolor{white}\ReplicaMean{}
& \cellcolor{MethodBlue}\HeadlineMean{} & \cellcolor{white}All three\\
\bottomrule
\end{tabular}}
\end{table}

\stitle{Key component analysis.}
\textbf{Single-family training shows uneven cross-readout transfer;
full \method{} gives the highest three-benchmark mean.}
In the single-family diagnostic of Table~\ref{tab:controlled-analysis}(a),
\textsc{Categorical-only} reaches \CategoricalControlOvos{} on OVO-S while
reducing OST to \CategoricalControlOst{}, whereas \textsc{Trajectory-only}
reaches \TrajectoryControlOvos{}/\TrajectoryControlOst{}/\TrajectoryControlUcs{}
on OVO-S/OST/UCS. The canonical \method{} model reaches
\CsaOvosUniform{}/\CsaOst{}/\CsaUcs{}, yielding the highest three-benchmark
mean of \HeadlineMean{}, compared with \CategoricalControlMean{} and
\TrajectoryControlMean{} for the two single-family variants. Source and task
coverage differ between the single-family variants, so this comparison
diagnoses cross-readout specialization rather than isolating answer interface
as a single causal factor (Appendix~\ref{app:single-family-control}).

\stitle{Comparison of joint-training objectives.}
With the joint data, chronological prefixes, native targets, and training
budget held fixed, ordinary token SFT obtains 51.8/60.7/49.8 on OVO-S/OST/UCS,
and natural turn weighting obtains 52.9/62.3/51.5
(Table~\ref{tab:joint-objectives}). \method{} improves all three endpoints,
with a mean of \HeadlineMean{} versus 54.10 and 55.57. This comparison supports the
balanced objective beyond joint data exposure under the tested setting.
Its OST advantage is 3.23 points over ordinary token SFT and 1.63 over natural
turn weighting; the corresponding OVO-S/UCS gains are 2.09/2.97 and 0.99/1.27.
The advantage over natural turn weighting repeats on every benchmark in each
of three paired seeds (Table~\ref{tab:objective-paired-seeds}). Mean paired
gains are $0.92\pm0.27$ on OVO-S, $1.43\pm0.19$ on OST, and $1.17\pm0.25$
on UCS (points, mean $\pm$ sample standard deviation).

\stitle{Comparison of answer formats.}
With fixed response-only readers and matched four-GPU runs, \textsc{Native} is higher on UCS
by \NativeInterfaceUcsDelta{} points and in the two-benchmark mean
(\NativeInterfaceMean{} versus \CommonInterfaceMean{}), while
\textsc{Canonical} is higher on OST by \CommonInterfaceOstDelta{} points.
This crossed result shows that answer format redistributes transfer rather
than improving every endpoint uniformly. We retain native targets to preserve
source task semantics without an additional answer-space transformation;
the control does not make native format a universally superior encoding.
The canonical wrapper preserves the original answer exactly. By
Corollary~\ref{cor:canonical} in Appendix~\ref{app:theory}, this is a lossless
serialization intervention, with unchanged target information.

\stitle{Comparison of training orders.}
Against \textsc{Shifted}, which keeps every turn block, target, and family
weight, chronological replay is ahead in both seed pairs, by
\OriginGapSeedOne{} and \OriginGapSeedTwo{} OST points (mean \OrderGain{}).
The shift changes both
the episode origin and the context preceding each target; the result supports
the temporal-origin alignment construction as a whole (RQ3). The
\CsaOstTrainGain{}-point OST improvement over the official base belongs to
the complete recipe, while this paired ordering intervention measures a
\OrderGain{}-point contribution from source chronology.

\stitle{Sensitivity analysis.}
Over a fixed replay pool, $\alpha=0.5$ gives the highest observed OVO-S--OST mean (61.20
versus 60.92 for $\alpha=1.0$), and retraining on four GPUs instead of eight
reproduces the three-benchmark mean within 0.34 points (RQ3;
Appendix~\ref{app:experimental-details}).

\begin{figure}[!t]
  \centering
  \includegraphics[height=5.05cm,trim=22 450 23 32,clip]{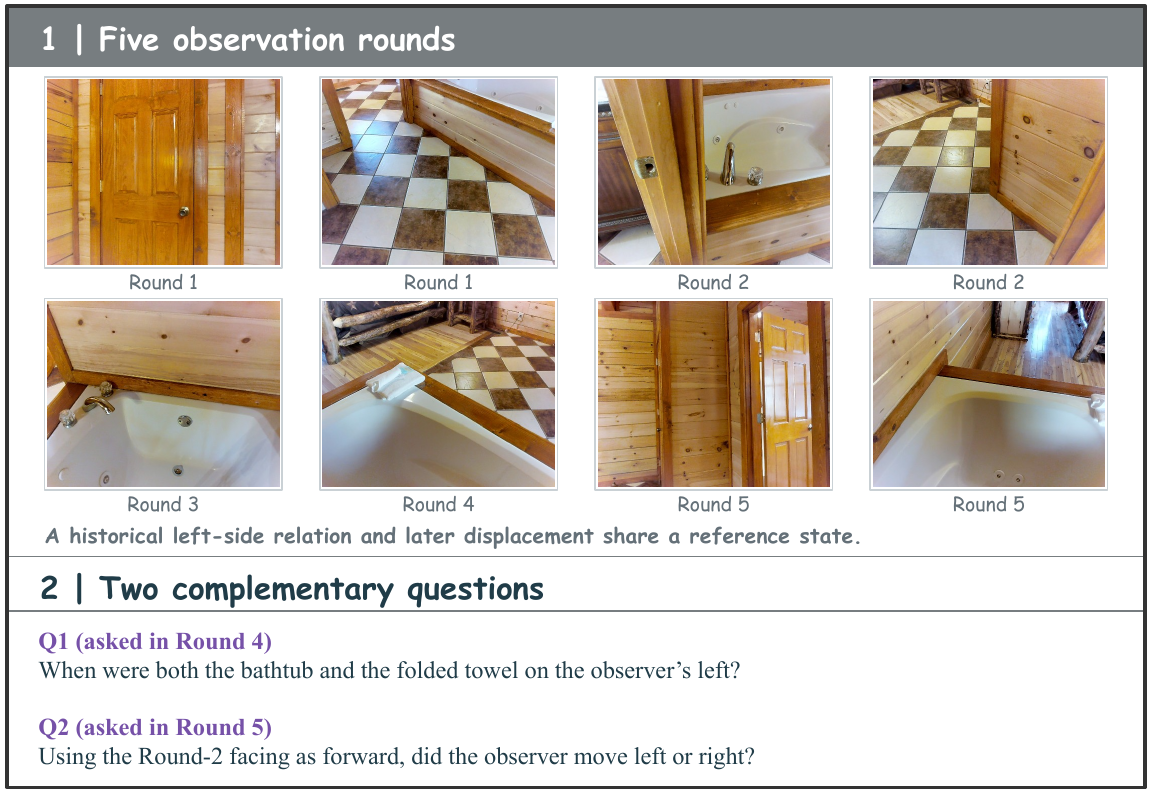}%
  \hspace{6pt}%
  \includegraphics[height=5.05cm,trim=22 13 23 414,clip]{figures/spatial_case.pdf}
  \caption{\textbf{One origin, two native readouts on an OST-Bench episode.}
  Left: the five observation rounds and two questions.  Q1 (Round~4): when
  were both the bathtub and the folded towel on the observer's left?
  Q2 (Round~5): using the Round-2 facing as forward, did the observer move
  left or right?  Right: \method{} answers Round~2 and Left, both correct,
  whereas the origin-shifted control (\textsc{Shifted}) returns Round~4 and
  Right under identical observations.  Teal marks annotated states, red the
  control's errors; the full-size episode is in
  Appendix~\ref{app:qualitative-case}.}
  \label{fig:case-main}
\end{figure}

\stitle{Qualitative analysis.}
Figure~\ref{fig:case-main} shows one episode whose two questions share a
reference state: a temporal retrieval and a displacement relative to an
earlier facing.  \method{} recovers both from the same origin, whereas
\textsc{Shifted} retrieves a later round and flips the displacement, the
failure the origin-shift control predicts (RQ3).
Appendix~\ref{app:qualitative-case} adds three more chronology cases, a
mirror-history comparison using historical preservation and trajectory
diagnostic models, and a fixed-observation L4 layout failure (RQ4).
The historical Preservation and Trajectory models used for qualitative
analysis are specified separately in the appendix.

\subsection{Observation Access as an Orthogonal Factor (RQ4)}
\label{sec:access-decomposition}

\begin{table}[!htb]
\caption{\textbf{Training--access decomposition on OVO-S} (\%).
Levels average their three family accuracies; Overall averages all twelve.
Summary rows separate training and access effects, using unrounded scores.}
\label{tab:ovos-system}
\centering
\small
\setlength{\tabcolsep}{5pt}
\begin{tabular*}{\textwidth}{@{\extracolsep{\fill}}llcccccc}
\toprule
Checkpoint & Access & L1 & L2 & L3 & L4 & Overall & Frames\\
\midrule
Official base & Uniform-128 & \BaseOvosLone{} & \BaseOvosLtwo{}
& \BaseOvosLthree{} & \BaseOvosLfour{} & \BaseOvosUniform{} & 128\\
\method{} & Uniform-128 & \CsaOvosUniformLone{} & \CsaOvosUniformLtwo{}
& \CsaOvosUniformLthree{} & \CsaOvosUniformLfour{}
& \CsaOvosUniform{} & 128\\
\rowcolor{MethodBlue}
\textbf{\method{}} & \textbf{Adaptive} & \CsaOvosLone{} & \CsaOvosLtwo{}
& \CsaOvosLthree{} & \CsaOvosLfour{} & \textbf{\CsaOvosSupport{}}
& \MeanFrames{}\\
\midrule
\multicolumn{2}{l}{Training effect (Uniform-128 $-$ base)}
& \TrainLone{} & \TrainLtwo{} & \TrainLthree{} & \TrainLfour{}
& +\CsaOvosTrainGain{} & --\\
\multicolumn{2}{l}{Access effect (adaptive $-$ Uniform-128)}
& \AccessLone{} & \AccessLtwo{} & \AccessLthree{} & \AccessLfour{}
& +\CsaOvosPolicyGain{} & --\\
\bottomrule
\end{tabular*}
\end{table}

\begin{samepage}
\stitle{Decomposition of the system score.}
\textbf{Adaptive access adds +\CsaOvosPolicyGain{} points beyond the
+\CsaOvosTrainGain{}-point training gain.} Table~\ref{tab:ovos-system}
separates the two effects under the same scoring rule. Training improves
all four levels, including L4 by \CsaOvosLfourGain{} points. Adaptive access
adds most on L1 (\AccessLone{}) and also improves L4 (\AccessLfour{}),
using \MeanFrames{} frames on average. L4 remains the hardest level.
Appendix~\ref{app:access-effects} reports the base-access and Uniform-81 controls.
\par
\end{samepage}

%% file: sections/06_conclusion.tex
\section{Conclusion}

\method{} replays heterogeneous spatial supervision in source chronology,
retains native targets, and improves OVO-S, OST, and UCS together
under matched observation access. We call the uneven
cross-readout transfer observed in single-family training \emph{interface
lock-in}, an empirical pattern rather than an isolated internal mechanism.
The single-family comparison retains the tradeoff: categorical-only
training favors OVO-S, trajectory-only training gives little OST Overall
gain, and full \method{} has the highest three-benchmark mean. Objective controls
show further gains beyond ordinary joint SFT. The full recipe
gains \CsaOstTrainGain{} OST points over the base; the paired cyclic-shift
control isolates a smaller, repeatable \OrderGain{}-point contribution from
preserving source chronology as a whole. Native and canonical answer formats
redistribute gains across UCS and OST rather than establishing a universally
best encoding. OVO-S improves across all four levels, including
+\CsaOvosLfourGain{} points on L4, which remains the hardest subset. These findings
support ChronoReplay as a simple recipe for balanced cross-interface
transfer without assigning its full gain to one component or claiming
dominance on every spatial subcategory.

%% file: sections/09_theory.tex
\section{Theoretical Analysis}
\label{app:theory}

This appendix formalizes the three arguments that motivate \method{}
(Sections~\ref{sec:native-refinement}, \ref{sec:origin-aligned-trajectories},
and \ref{sec:family-balanced-learning}).  Each result is stated at the
population level with an unrestricted readout, so it describes what the
objective asks of the representation; it does not certify that finite-data
token optimization attains it. The controls in
Section~\ref{sec:controlled-analysis} examine the empirical relevance of
these properties.
Throughout, $C$ is a causal prefix drawn from the replay pool, $Q$ its query,
$\kappa$ the native interface, $Y^\kappa=A_\kappa(C,Q)$ the native target,
$Z=F_\theta(C)$ the representation, and $H(\cdot\mid\cdot)$ and
$I(\cdot;\cdot\mid\cdot)$ denote conditional entropy and conditional mutual
information.

\subsection{Native Readouts as Sufficiency Constraints}

\begin{proposition}[Native risk and query-conditioned prediction]
\label{prop:sufficiency}
For any representation $F_\theta$ and any interface $\kappa$, the population
risk of Eq.~\ref{eq:native-interface-risk} satisfies
\begin{equation}
\inf_{G_\kappa}\mathcal R_\kappa(F_\theta,G_\kappa)
= H(Y^\kappa\mid Z,Q)
= H(Y^\kappa\mid C,Q) + I(Y^\kappa; C\mid Z,Q),
\label{eq:risk-decomposition}
\end{equation}
where the infimum is over all conditional distributions of $Y^\kappa$ given
$(Z,Q)$. For each query $q$ in the support of $Q$,
$H(Y^\kappa\mid Z,Q=q)=0$ implies perfect prediction from $(Z,q)$
on that conditional support.
\end{proposition}

\begin{proof}
For fixed $(Z,Q)$,
$\mathbb E[-\log G_\kappa(Y^\kappa\mid Z,Q)]
= H(Y^\kappa\mid Z,Q)+\mathbb E\,\mathrm{KL}\big(P(Y^\kappa\mid Z,Q)\,\|\,G_\kappa(\cdot\mid Z,Q)\big)$,
and the divergence vanishes exactly when $G_\kappa$ equals the true
conditional, which is admissible for an unrestricted readout; this gives the
first equality.  Since $Z$ is a function of $C$,
$H(Y^\kappa\mid C,Q)=H(Y^\kappa\mid C,Z,Q)$, and the chain rule gives
$H(Y^\kappa\mid Z,Q)-H(Y^\kappa\mid C,Z,Q)=I(Y^\kappa;C\mid Z,Q)\ge 0$, which
is the second equality. Zero conditional entropy means that the conditional
target distribution is a point mass almost surely, yielding the
query-conditioned prediction statement.
\end{proof}

We use this result to motivate supervision across heterogeneous query
interfaces; cross-interface transfer is evaluated empirically in
Section~\ref{sec:experiments}.

\begin{corollary}[Lossless re-encoding preserves target information]
\label{cor:canonical}
Let $\tilde Y=g(Y^\kappa)$, where $g$ is invertible on the target support.
Then for every representation,
$H(\tilde Y\mid Z,Q)=H(Y^\kappa\mid Z,Q)$ and
$I(\tilde Y;C\mid Z,Q)=I(Y^\kappa;C\mid Z,Q)$.
The two encodings impose the same target-sufficiency requirement at the
population level.
\end{corollary}

\begin{proof}
Each target is a deterministic function of the other. Conditional entropy
cannot increase under either deterministic mapping, giving equality.
Applying the conditional data-processing inequality in both directions
likewise gives equality of the mutual information terms.
\end{proof}

Our canonical control wraps the original answer string in a JSON
\texttt{value} field, so the target remains recoverable exactly. It tests
transfer under a lossless change in output serialization. The crossed
OST/UCS result in Table~\ref{tab:controlled-analysis} therefore concerns
finite-model training and readout behavior; it cannot be explained by a loss
of target information in the wrapper. Equality at the population level does
not imply identical token losses or optimization trajectories.

\subsection{Cyclic Shifts Change the Conditioning Context}

Let $O$ denote the identity of an episode's first turn block, $C_j$ the
legal chronological context for target $j$, and $C'_j$ its preceding context
after the full episode is cyclically shifted
(Appendix~\ref{app:origin-shift-construction}). The shift preserves each
block's local content, but $C'_j$ can omit earlier blocks and include blocks
that originally occurred later. Knowing $O$ therefore need not recover
$C_j$ from $C'_j$.

\begin{proposition}[Risk under context ambiguity]
\label{prop:origin}
Let $Y=A(C_j,Q)$ be a deterministic native target. With an unrestricted
readout, the population log-loss floor under shifted conditioning is
\begin{equation}
\inf_G\mathbb E[-\log G(Y\mid C'_j,Q)]
=H(Y\mid C'_j,Q)
=I(Y;O\mid C'_j,Q)+H(Y\mid C'_j,O,Q).
\label{eq:shifted-context-risk}
\end{equation}
This floor exceeds $H(Y\mid C_j,Q)=0$ if, on a set of shifted inputs
$(C'_j,Q)$ of positive probability, compatible chronological histories
induce at least two distinct targets with positive conditional probability.
\end{proposition}

\begin{proof}
Expected log loss is conditional entropy plus the expected KL divergence
from the true conditional distribution to $G$. The divergence vanishes at
$G=P(Y\mid C'_j,Q)$, giving the first equality. The second is the chain
rule for conditional mutual information. Under the stated ambiguity
condition, the target distribution has positive entropy on a set of
positive probability, so $H(Y\mid C'_j,Q)>0$. Chronological conditioning
has zero entropy because $Y$ is a function of $(C_j,Q)$.
\end{proof}

The residual term $H(Y\mid C'_j,O,Q)$ captures ambiguity that remains even
when the origin is known, including ambiguity from missing history.
Cyclic shifts need not satisfy the ambiguity condition: later observations
introduced into a prefix may preserve or reveal the answer. The proposition
therefore gives a conditional account of context corruption; it does not
imply a strict penalty for every shift or determine downstream accuracy.
The measured advantage of chronological replay (\OrderGain{} OST points
averaged over two seeds in Table~\ref{tab:origin-seed-details}) supports
preserving source order and legal causal prefixes together, without isolating
a latent origin representation.

\subsection{What Family Balance Does and Does Not Control}

\begin{proposition}[Invariances of the balanced objective]
\label{prop:balance}
Let $\mathcal L_P$ and $\mathcal L_T$ be the family risks of
Eq.~\ref{eq:family-risks} and
$\mathcal L_{\mathrm{replay}}=(1-\lambda)\mathcal L_P+\lambda\mathcal L_T$
with $\lambda=\alpha/(1+\alpha)$.  Then:
(i) replacing $\mathcal D_P$ or $\mathcal D_T$ by a family with the same
within-family empirical distribution but a different number of sequences or
supervised turns leaves $\nabla_\theta\mathcal L_{\mathrm{replay}}$
unchanged, while $\alpha$ sets the explicit ratio of the two family weights;
(ii) $\mathcal L_T$ is invariant to whether the supervised transitions of an
episode are grouped into one sequence or presented as separate examples,
provided each transition keeps its causal prefix;
(iii) any common positive rescaling of all example weights, such as the one
that sets the mean weight to one (Appendix~\ref{app:experimental-details}),
leaves the minimizers of $\mathcal L_{\mathrm{replay}}$ and the direction of
its gradient unchanged.
\end{proposition}

\begin{proof}
(i) Each family risk is a weighted mean, so it depends only on the empirical
distribution of $(\ell,w)$ within the family and not on the number of terms;
the same holds for its gradient.  (ii) By Eq.~\ref{eq:family-risks},
$\mathcal L_T$ is the plain average of $\ell_{ij}^T$ over all supervised
transitions $(i,j)\in\mathcal T$, and each $\ell_{ij}^T$ is computed from the
prefix $c_{ij}$ alone (Eq.~\ref{eq:appendix-native-token-loss}); grouping
changes only how the same terms are batched.  (iii) A common factor $s>0$ on
all weights cancels between the numerator and denominator of each family
risk when applied before normalization, and multiplies the objective by $s$
when applied after; in either case the arg-min and the gradient direction are
unchanged.
\end{proof}

Proposition~\ref{prop:balance} is why the family-mass sweep of
Table~\ref{tab:family-mass-sweep} is a sweep over a single scalar: with
normalization in place, changing $\alpha$ changes the relative influence of
the two families and nothing else.  It also makes explicit what balance does
not control: coverage of states within the trajectory family still grows with
episode length, and no choice of $\alpha$ recovers information that a
supervision family never supplied, which is the content of
Proposition~\ref{prop:sufficiency}.

%% file: sections/08_appendix.tex
\section{Evaluation and Attribution Protocol}
\label{app:protocol}

\subsection{Attribution Framework}

Performance depends jointly on the trained model and its observation access.
Let $J(\theta,\pi)$ denote benchmark performance with model
parameters $\theta$ and observation policy $\pi$.  We separate the two effects
as
\begin{equation}
\Delta_{\mathrm{train}}
=
J(\theta^\star,\pi_0)-J(\theta_0,\pi_0),
\qquad
\Delta_{\mathrm{access}}
=
J(\theta^\star,\pi_1)-J(\theta^\star,\pi_0),
\label{eq:attribution-effects}
\end{equation}
where $\theta_0$ is the official base model, $\theta^\star$ is the model
trained with \method{}, $\pi_0$ is the standard observation policy, and
$\pi_1$ is adaptive access (Appendix~\ref{app:temporal-support}).

The first contrast changes the trained model under identical observation access.
The second changes observation access while holding the model parameters fixed.
The complete system evaluates $J(\theta^\star,\pi_1)$, while the two controlled
contrasts separate the training and access effects.

We examine design choices through matched training interventions. The
origin-shift intervention preserves the same multiset of turn blocks and
supervision weights while changing their episode-level chronology and the
accumulated context preceding each target. Initialization and optimization
settings are held fixed. The target-interface
experiment changes the answer representation while preserving the causal
histories and training configuration. These comparisons probe sensitivity
to the corresponding design choices of \method{}.

\subsection{Benchmark Protocols}

\stitle{OVO-S-Bench.}
Our fixed OVO-S evaluation uses only observations available by the query
boundary.  Uniform access to 128 frames defines $\pi_0$ and is used to measure
$\Delta_{\mathrm{train}}$.  Adaptive access defines $\pi_1$ for the complete
system and selects one causal observation set before answer generation.
All our OVO-S main results and component controls use the same scoring rule:
family accuracy is $a_f=100k_f/N_f$ and Overall is
$\frac{1}{12}\sum_{f=1}^{12}a_f$. Each level averages its three families.
We aggregate the unrounded fractions and round only the reported result.
Table~\ref{tab:ovos-counts} gives the counts for the main results. The pooled accuracies
are 50.91, 55.69, and 60.88 for Base, Uniform-128, and Adaptive, respectively;
these are distinct from the reported family means of 48.98, 53.89, and 57.98.

Published external rows are quoted as references under their reported protocols.

\begin{table}[H]
\caption{Correct answers underlying the OVO-S main results.
Uniform-128 and Adaptive denote \method{}; Base uses Uniform-128.}
\label{tab:ovos-counts}
\centering
\small
\begin{tabular}{lrrrr}
\toprule
Family & $N_f$ & Base & Uniform-128 & Adaptive\\
\midrule
1.1 & 167 & 96 & 105 & 116\\
1.2 & 194 & 117 & 125 & 138\\
1.3 & 268 & 119 & 130 & 164\\
2.1 & 198 & 108 & 117 & 119\\
2.2 & 168 & 79 & 87 & 98\\
2.3 & 147 & 103 & 109 & 113\\
3.1 & 150 & 89 & 98 & 103\\
3.2 & 63 & 17 & 21 & 22\\
3.3 & 66 & 41 & 44 & 45\\
4.1 & 131 & 26 & 32 & 35\\
4.2 & 90 & 55 & 61 & 63\\
4.3 & 53 & 13 & 15 & 16\\
\midrule
Total & 1,695 & 863 & 944 & 1,032\\
\bottomrule
\end{tabular}
\end{table}

The access policy uses only the visible semantics of the query.  Its routing
rules and frame construction remain fixed throughout evaluation.  The
controlled comparison with uniform access to 81 frames measures the effect of
temporal allocation at approximately equal visual cost.

\stitle{OST-Bench.}
OST conversations are replayed in their original turn order.  At each turn,
the model receives the observations and interaction history available at that
point in the episode.  Evaluation uses a fixed split with disjoint training
and test scenes, containing 8,351 turns in total and 7,814 turns scored by the
official evaluator.

We report the complete official taxonomy: agent state, visible agent
information, spatial relations between the agent and objects, judgement,
temporal localization, counting, and estimation.  The detailed table further
decomposes these categories into the fifteen official question types.

\stitle{UCS-Bench.}
We evaluate all 8,114 questions through the official interface.  Each question
is conditioned on observations available at its associated timestamp.  The
four official categories are Trajectory \& Movement, Position \& Orientation,
Proximity \& Reachability, and Category \& Quantity, abbreviated as
Trajectory, Position, Proximity, and Category in our tables. Each category
has two subcategories: the twelve displayed scores comprise four category
aggregates and eight subcategory scores \citep{wang2026keep}.
All compared Qwen3-VL-8B
checkpoints use the same observation and scoring protocol.

\subsection{Scope of Comparison}

Published systems provide external reference points under their reported
observation settings.  Attribution claims are based on matched Qwen3-VL-8B
comparisons. Post-training is evaluated under fixed observation
access, adaptive access is evaluated with fixed model parameters, and method
components are evaluated through matched training interventions.  This
separation determines which conclusions can be drawn from each experimental
block.

\FloatBarrier

\section{Replay Supervision}
\label{app:replay-data}

\subsection{Composition}

\method{} combines two supervision families. The single-turn preservation
pool contains 30,635 categorical and 2,545 numerical examples. The trajectory
pool contains ordered episodes with metric, counting, judgement, and temporal
targets. The families differ in source and interaction structure as well as
readout composition.

\begin{table}[H]
\caption{Composition and relative loss mass of the Causal Replay pool at
$\alpha=0.5$. A supervised state is one causal prefix with a native target.
Mass denotes the sum of training weights, normalized over the whole pool.}
\label{tab:replay-pool}
\centering
\small
\setlength{\tabcolsep}{4pt}
\begin{tabular*}{\textwidth}{@{\extracolsep{\fill}}llrrc}
\toprule
Family & Targets & Sequences & Supervised states & Loss mass \\
\midrule
Preservation & Categorical & 30,635 & 30,635 & $1/3$\\
 & Numerical & 2,545 & 2,545 & $1/3$\\
Trajectory & Mixed spatial & \TrajectoryRows{} & \TrajectoryTurns{} & $1/3$\\
\bottomrule
\end{tabular*}
\end{table}

The complete data contain \ReplayRows{} serialized sequences.  The trajectory
family consists of \TrajectoryRows{} episodes and
\TrajectoryTurns{} supervised states.  An episode may also contain observation
turns that establish the spatial context for later targets.  These turns enter
the causal prefix and receive no direct loss.

\stitle{Preservation sources and weights.}
The categorical pool comprises 9,000 source-annotated question--answer pairs,
each presented with dense, contextual, and long-context observations
(27,000 rows), plus 3,635 source-verified categorical examples under uniform causal
observation access. The numerical pool comprises 945 ARKitScenes
camera-rotation targets \citep{baruch2021arkitscenes} and 800 targets each for camera displacement and
camera--object distance from corrected camera-geometry records.
The trajectory pool uses the released OST training episodes.

Let $s=13{,}187/12{,}635$. Each of the three views of a categorical pair has
weight $s/3$, and each of the remaining 3,635 categorical examples has weight
$s$. Every numerical example has weight $13{,}187/2{,}545$. A trajectory
episode with $m_i$ supervised turns has serialized weight
$13{,}187\,m_i/45{,}006$ and averages its supervised turn losses.
Thus the categorical, numerical, and trajectory weight sums are each
13,187, and the mean serialized-example weight is one. These are loss
coefficients; they do not assert equal measured gradient norms.

The preservation and trajectory families retain their original questions and
answer spaces.  Categorical targets preserve candidate semantics, metric
targets preserve magnitude and units, and temporal targets preserve event
indices or timestamps.  \method{} therefore combines the two families
through a shared model while maintaining the distinctions supplied by each
source.

\subsection{Episode Construction}

Each trajectory is replayed in its original temporal order.  At turn $j$, the
model receives the observations available by that turn and all preceding
interaction:
\begin{equation}
c_{ij}
=
\left(
h_{ij};
q_{i1},y_{i1},\ldots,q_{i,j-1},y_{i,j-1}
\right).
\label{eq:serialized-causal-context}
\end{equation}
The current target $y_{ij}$ follows this prefix as the next assistant response.
Later observations and responses remain outside $c_{ij}$.

This construction preserves the trajectory origin and the state accumulated
before each target.  Object identity, observer pose, reference frame, and
source units remain consistent throughout an episode.  The resulting data
contain 47,096 turns, comprising 45,006 supervised states and 2,090 observation
turns that provide context.

Keeping an episode intact also preserves dependencies across turns.  A metric
displacement is interpreted relative to the preceding observer state, while a
temporal answer is interpreted relative to the events already observed.
Training therefore associates each target with the causal context under which
its spatial meaning is defined.  Algorithm~\ref{alg:chronoreplay} summarizes
the construction together with the family-balanced objective of
Section~\ref{sec:family-balanced-learning}.

\begin{algorithm}[H]
\caption{\method{}: Causal Replay construction and family-balanced learning.}
\label{alg:chronoreplay}
\small
\begin{tabbing}
\hspace{1.6em}\=\hspace{1.6em}\=\hspace{1.6em}\=\kill
\textbf{Input:} base VLM $\theta_0$; preservation pool $\mathcal D_P$ with
weights $w_r$; multi-round sources; trajectory mass $\alpha$\\
\textbf{Output:} trained parameters $\theta$\\[2pt]
\textit{Stage I: Causal Replay construction}\\
1:\> $\mathcal D_T\leftarrow\emptyset$\\
2:\> \textbf{for} each multi-round source episode $i$, turns kept in original
order \textbf{do}\\
3:\>\> \textbf{for} $j=1,\ldots,m_i$ \textbf{do}\\
4:\>\>\> $c_{ij}\leftarrow(h_{ij};\,q_{i1},y_{i1},\ldots,q_{i,j-1},y_{i,j-1})$
\hfill $\triangleright$ causal prefix, Eq.~\ref{eq:serialized-causal-context}\\
5:\>\>\> keep $y_{ij}$ in its native format; record its interface
$\kappa_{ij}$ \hfill $\triangleright$ Section~\ref{sec:native-refinement}\\
6:\>\> \textbf{end for}\\
7:\>\> $\mathcal D_T\leftarrow\mathcal D_T\cup
\{(c_{ij},q_{ij},y_{ij},\kappa_{ij})_{j=1}^{m_i}\}$
\hfill $\triangleright$ Eq.~\ref{eq:replay-episode}\\
8:\> \textbf{end for}\\[2pt]
\textit{Stage II: family-balanced learning}\\
9:\> $\mathcal L_P\leftarrow\sum_{r}w_r\ell_r^P/\sum_r w_r$;\quad
$\mathcal L_T\leftarrow|\mathcal T|^{-1}\sum_{(i,j)\in\mathcal T}\ell_{ij}^T$
\hfill $\triangleright$ Eq.~\ref{eq:family-risks}\\
10:\> $\lambda\leftarrow\alpha/(1+\alpha)$\\
11:\> $\theta\leftarrow\arg\min_{\theta}\,(1-\lambda)\,\mathcal L_P
+\lambda\,\mathcal L_T$, starting from $\theta_0$
\hfill $\triangleright$ Eq.~\ref{eq:causal-replay-objective}; one epoch,
Table~\ref{tab:training-config}\\
12:\> \textbf{return} $\theta$
\end{tabbing}
\end{algorithm}

\subsection{Scene Separation}

We separate training and evaluation by physical scene.  A scene may appear
under different frame rates, crops, filenames, or video encodings, so media
files derived from the same scan are assigned a common scene identity.  This
identity is established before replay examples are constructed.

Our fixed OVO-S evaluation and the UCS evaluation share no physical scenes
with the training data.
For OST, we train on its released training data and evaluate on a fixed split
with disjoint scenes.  We withhold \OstWithheldScenes{} overlapping ScanNet
scenes \citep{dai2017scannet}, while retaining every official OST question
type in the evaluation set.

Replay construction uses only annotations, observations, timestamps, reference
frames, and answers from the training sources.  Evaluation questions and
answers play no role in example construction, family weighting, or target
selection.  The same scene separation applies to every model variant and
matched intervention.

\section{Optimization and Controlled Comparisons}
\label{app:experimental-details}

\subsection{Training Settings}

The primary model and its original controls are initialized independently from
the same Qwen3-VL-8B-Instruct parameters. Training updates the language model
and visual merger while keeping the visual encoder fixed. These original runs
use one epoch of supervised learning with the configuration in
Table~\ref{tab:training-config}. The single-family diagnostic in
Appendix~\ref{app:single-family-control} compares the canonical method with
categorical-only and trajectory-only variants.
The headline scores use one random seed. The natural-turn versus full-objective
comparison covers three paired training seeds (Table~\ref{tab:objective-paired-seeds}).
The comparison of chronological and shifted replay is repeated with two random
seeds, with each pair sharing its initialization and optimization settings.

\begin{table}[H]
\caption{Shared optimization settings for \method{} and its matched
controls.}
\label{tab:training-config}
\centering
\small
\begin{tabular*}{\textwidth}{@{\extracolsep{\fill}}ll}
\toprule
Setting & Value\\
\midrule
Base model & Qwen3-VL-8B-Instruct\\
Updated components & Language model and visual merger\\
Epochs / global batch size & 1 / 32\\
Learning rate & $2\times10^{-6}$ \\
Schedule & Cosine decay with 3\% warmup\\
Optimizer & AdamW, $\beta=(0.9,0.999)$\\
Weight decay & $0.01$\\
Numerical precision & bfloat16\\
Maximum sequence length & 32,768 tokens\\
Maximum processed pixels & 200,704 per visual item\\
Relative trajectory mass & $\alpha=\TrajectoryMass{}$\\
\bottomrule
\end{tabular*}
\end{table}
\FloatBarrier

The same tokenizer, image processor, and conversation template are used
across runs.  Loss is applied to supervised assistant responses, with
observations and preceding interaction providing causal context.
Appendix~\ref{app:reproducibility} specifies the model revision and execution
details.

\subsection{Loss Normalization}

We normalize supervision at the answer, transition, and family levels.
For an answer sequence $y=(y_1,\ldots,y_N)$ following causal context $c$
and query $q$, the native loss is
\begin{equation}
\ell_\theta(c,q,y,\kappa)
=-\frac{1}{N}\sum_{n=1}^{N}
\log p_\theta(y_n\mid c,q,\kappa,y_{<n}).
\label{eq:appendix-native-token-loss}
\end{equation}
Averaging over answer tokens gives each target a contribution normalized
by its length.

For episode $i$ with $m_i$ supervised turns, we average the turn losses
and assign the episode a weight proportional to $m_i$, which is the same as
giving every supervised transition one common weight within the trajectory
family while retaining its complete causal context.  Longer episodes
therefore contribute in proportion to their number of supervised turns.

Preservation examples retain the weights supplied by their source.
The total trajectory weight is then scaled to $\alpha$ times the total
preservation weight, and a common rescaling sets the mean example weight
to one.  This implements the family risks defined in the main text.

For a fixed set of causal contexts and targets, this weighted empirical
objective is invariant to grouping turns into episodes or separate
examples.  The grouping used during optimization determines how those
contexts are presented.

\stitle{Joint-training objective comparison.}
Table~\ref{tab:joint-objectives} compares three objectives on the same joint
replay data, with chronological prefixes, native targets, training budget,
and evaluation access held fixed. Ordinary token SFT uses the standard loss
over supervised assistant tokens. Natural turn weighting retains answer-length
normalization and within-family weighting, with family mass determined by
the resulting turn weights. \method{} additionally sets the relative family
mass explicitly through $\alpha=0.5$.

\begin{table}[H]
\caption{Joint-training objective comparison (\%). OVO-S uses Uniform-128;
OST and UCS use the same fixed protocols as the main results. Mean is the
unweighted average of the three benchmark scores.}
\label{tab:joint-objectives}
\centering
\small
\begin{tabular*}{\textwidth}{@{\extracolsep{\fill}}lcccc}
\toprule
Training objective & OVO-S & OST & UCS & Mean\\
\midrule
Base & \BaseOvosUniform{} & \BaseOst{} & \BaseUcs{} & \BaseMean{}\\
Ordinary token SFT & 51.8 & 60.7 & 49.8 & 54.10\\
Natural turn weighting & 52.9 & 62.3 & 51.5 & 55.57\\
\rowcolor{MethodBlue}
\method{} & \textbf{\CsaOvosUniform{}} & \textbf{\CsaOst{}}
& \textbf{\CsaUcs{}} & \textbf{\HeadlineMean{}}\\
\bottomrule
\end{tabular*}
\end{table}
\FloatBarrier

Both ordinary token SFT and natural turn weighting improve all three
benchmarks over the base. Family-balanced learning provides a further gain
on each benchmark. These comparisons support the combined weighting design
without separately isolating every level of normalization.

\stitle{Paired-seed replication.}
Table~\ref{tab:objective-paired-seeds} repeats the natural-turn versus
full-objective comparison across three paired training seeds. The first pair
corresponds to the original scores in Table~\ref{tab:joint-objectives};
the other two extend the comparison to additional seeds. The full objective improves all
three benchmarks in every pair. The table reports the sample standard
deviation of the paired differences to describe variation across these runs.

\begin{table}[H]
\caption{Three-seed replication of the objective comparison (\%). Natural:
natural turn weighting; Full: \method{}; $\Delta$: Full minus Natural in
percentage points. Pair indices enumerate runs. Differences and summaries
use the displayed scores; summary values are rounded after aggregation.}
\label{tab:objective-paired-seeds}
\centering
\small
\setlength{\tabcolsep}{3pt}
\begin{tabular*}{\textwidth}{@{\extracolsep{\fill}}l*{9}{r}}
\toprule
& \multicolumn{3}{c}{OVO-S} & \multicolumn{3}{c}{OST} & \multicolumn{3}{c}{UCS}\\
\cmidrule(lr){2-4}\cmidrule(lr){5-7}\cmidrule(lr){8-10}
Pair & Natural & Full & $\Delta$ & Natural & Full & $\Delta$ & Natural & Full & $\Delta$\\
\midrule
1 & 52.90 & 53.89 & +0.99 & 62.30 & 63.93 & +1.63 & 51.50 & 52.77 & +1.27\\
2 & 52.71 & 53.33 & +0.62 & 62.58 & 63.83 & +1.25 & 51.68 & 52.56 & +0.88\\
3 & 53.02 & 54.17 & +1.15 & 62.14 & 63.56 & +1.42 & 51.39 & 52.74 & +1.35\\
\midrule
\rowcolor{MethodBlue}
Mean & 52.88 & 53.80 & \textbf{+0.92} & 62.34 & 63.77 & \textbf{+1.43}
& 51.52 & 52.69 & \textbf{+1.17}\\
SD of $\Delta$ & -- & -- & 0.27 & -- & -- & 0.19 & -- & -- & 0.25\\
\bottomrule
\end{tabular*}
\end{table}
\FloatBarrier

\subsection{Relative Weight of Trajectory Supervision}
\label{app:family-mass-sweep}

The family-mass sweep is a separate matched set of runs. We vary $\alpha$
while holding training examples, chronological order, initialization, and
optimization settings fixed within this set. The $\alpha=0.5$ row is an
independent replication under the sweep protocol, distinct from the headline
checkpoint.
Table~\ref{tab:family-mass-sweep} reports OVO-S under the same adaptive
observation policy and OST under the complete evaluation protocol.
For comparison under adaptive access, the headline model scores
\CsaOvosSupport{} on OVO-S and \CsaOst{} on OST.
The final column is the arithmetic mean of the two benchmark scores.

\begin{table}[H]
\caption{Sensitivity to the relative weight of trajectory supervision in a
separate matched run set (\%). OVO-S uses the fixed adaptive access policy.}
\label{tab:family-mass-sweep}
\centering
\small
\begin{tabular*}{\textwidth}{@{\extracolsep{\fill}}cccc}
\toprule
$\alpha$ & OVO-S & OST & Mean\\
\midrule
0.25 & 57.56 & 62.73 & 60.15\\
\rowcolor{MethodBlue}
0.50 & \textbf{58.27} & \textbf{64.13} & \textbf{61.20}\\
1.00 & 57.80 & 64.04 & 60.92\\
\bottomrule
\end{tabular*}
\end{table}
\FloatBarrier

Among the tested settings, $\alpha=0.50$ gives the highest score on both
benchmarks.  This comparison supports an intermediate allocation of
supervision between single-turn preservation and trajectory learning.

\subsection{Target Interface}
\label{app:matched-interface}

This comparison examines how answer representation affects transfer.
Two models share the same 39,561 training examples, initialization,
random seed, global batch size, training schedule, and family weights.
The native variant retains the original source responses. The canonical
variant serializes each complete original answer string as
\texttt{\{"value": "original answer"\}}, retaining units, precision, and
wording. JSON decoding recovers the original target exactly. Observations,
questions, supervised-turn positions, row order, and training weights are
unchanged; the intervention changes the assistant's output serialization.

Evaluation converts each response into the input expected by the official
scorer using a fixed reader for its trained format.  Each reader operates
on the generated response alone.  All 8,351 OST responses and all 8,114 UCS
responses yield an extracted value; the official OST scorer evaluates its
designated 7,814 turns.  Detailed extraction rules are provided in
Appendix~\ref{app:reproducibility}.  The comparison covers OST and UCS,
whose generated answers can be evaluated through these readers.
OVO-S uses greedy generation with deterministic option-label extraction.

Canonical supervision obtains \CommonInterfaceOst{} on OST and
\CommonInterfaceUcs{} on UCS.  Native supervision obtains
\ReplicaOst{} and \ReplicaUcs{}, respectively.  Canonical targets improve
OST by 1.48 points, while native targets improve UCS by 2.92 points.
Native supervision consequently yields the higher mean across the two
benchmarks.  This crossed response indicates that the training interface
affects the distribution of transfer across tasks.

\subsection{Chronological Alignment}
\label{app:origin-shift-construction}

We examine the contribution of chronology by cyclically shifting the
turn blocks within each episode.  Let $b_i$ denote the total number of
turn blocks, including turns that provide observation context.  For
$b_i>1$, a deterministic offset $d_i\in\{1,\ldots,b_i-1\}$ defines
\begin{equation}
\sigma_i(j)
=1+\bigl((j+d_i-1)\bmod b_i\bigr).
\label{eq:origin-shift-permutation}
\end{equation}
The offset is determined by the episode identifier.  Episodes containing
one block retain their original order.

Each block carries its observations, query, response, and supervision
weight into the shifted sequence.  The two variants therefore share
the same local content and total family weight.  They also share
initialization and optimization settings.  The shift changes the starting
point of the episode and the accumulated context preceding each target.

We repeat the comparison with two random seeds.  Within each pair,
chronological and shifted replay share the same initialization, training
data, and optimization settings.  Both models are evaluated on identical
conversations in their original temporal order.
Table~\ref{tab:origin-seed-details} reports their OST performance.

\begin{table}[H]
\caption{Chronological and shifted replay on OST (\%).  Each row compares
models trained with the same seed and evaluated on the same scenes.}
\label{tab:origin-seed-details}
\centering
\small
\begin{tabular*}{\textwidth}{@{\extracolsep{\fill}}lcc}
\toprule
Training order & Run 1 & Run 2\\
\midrule
Cyclic shift & \ShiftedOst{} & \ShiftedSeedTwoOst{}\\
\rowcolor{MethodBlue}
Chronological & \textbf{\CsaOst{}} & \textbf{\ChronoSeedTwoOst{}}\\
\bottomrule
\end{tabular*}
\end{table}
\FloatBarrier

Chronological replay is ahead by \OriginGapSeedOne{} and
\OriginGapSeedTwo{} points in the two comparisons, a mean advantage of
\OrderGain{} points.  Because the shift also changes the context preceding
each target, the comparison tests chronological alignment as a whole rather
than the episode origin alone.

\subsection{Training-Topology Replication}
\label{app:topology-replication}

Table~\ref{tab:topology-replication} tests whether the cross-interface result
depends on the synchronized training topology.  Both rows start from the same
pinned official checkpoint and use the exact same 39,561 training rows, seed,
optimizer, global batch of 32, and one-epoch schedule.  The replication changes
world size from eight to four and compensates by changing gradient accumulation
from two to four.

\begin{table}[htbp]
\caption{Fresh training-topology replication under complete model-only
evaluation (\%), together with the frozen OVO-S system protocol.  Observation
access and scorers are fixed within each reported column.}
\label{tab:topology-replication}
\centering
\footnotesize
\resizebox{\textwidth}{!}{%
\begin{tabular}{lcccccc}
\toprule
Run & GPUs & OVO-U128 & OVO-Adaptive & OST & UCS & Model-only mean\\
\midrule
Headline checkpoint & 8 & \CsaOvosUniform{} & \CsaOvosSupport{}
& \CsaOst{} & \CsaUcs{} & \HeadlineMean{}\\
\rowcolor{MethodBlue}
Fresh replication & 4 & \ReplicaOvosUniform{}
& \ReplicaOvosSupport{} & \ReplicaOst{} & \ReplicaUcs{} & \ReplicaMean{}\\
\bottomrule
\end{tabular}}
\end{table}
\FloatBarrier

The rerun preserves positive transfer over the official base on all three
endpoints and improves the model-only mean by 0.34 points. Relative to the
headline checkpoint, OVO-S Uniform-128 increases by 0.21 points, UCS by 1.49,
and OST decreases by 0.69. This supports robustness to the distributed
training topology, but does not constitute an independent-seed variance
estimate. The fixed adaptive access policy raises the replica from
\ReplicaOvosUniform{} to \ReplicaOvosSupport{}, an access gain of
\ReplicaOvosPolicyGain{} points and a close replication of the headline
\CsaOvosSupport{} system score.

\subsection{Single-Family Supervision Control}
\label{app:single-family-control}

Table~\ref{tab:single-family-control} gives the single-family diagnostic
summarized in Table~\ref{tab:controlled-analysis}(a). The \method{} row is the
canonical model reported in Tables~\ref{tab:ovos-main}--\ref{tab:ucs-main}.
This categorical-only arm differs from the full preservation pool, which
also includes numerical examples (Table~\ref{tab:replay-pool}).

\begin{table}[H]
\caption{Single-family diagnostic on OVO-S, OST, and UCS (\%).
The \method{} row uses the canonical model from Tables~\ref{tab:ovos-main}--\ref{tab:ucs-main}.
Mean averages the three benchmarks under the evaluation protocols described above.}
\label{tab:single-family-control}
\centering
\small
\begin{tabular*}{\textwidth}{@{\extracolsep{\fill}}lcccc}
\toprule
Training & OVO-S & OST & UCS & Mean\\
\midrule
Official base & \BaseOvosUniform{} & \BaseOst{} & \BaseUcs{} & \BaseMean{}\\
Categorical-only & \textbf{\CategoricalControlOvos{}} & \CategoricalControlOst{} & \CategoricalControlUcs{} & \CategoricalControlMean{}\\
Trajectory-only & \TrajectoryControlOvos{} & \TrajectoryControlOst{} & \TrajectoryControlUcs{} & \TrajectoryControlMean{}\\
\rowcolor{MethodBlue}
\textbf{\method{}} & \textbf{\CsaOvosUniform{}} & \textbf{\CsaOst{}} & \textbf{\CsaUcs{}} & \textbf{\HeadlineMean{}}\\
\bottomrule
\end{tabular*}
\end{table}

Categorical-only training leads OVO-S but falls below the base on OST;
trajectory-only gives little OST Overall gain. Both single-family arms
improve UCS over the base. Canonical \method{} achieves the highest three-benchmark
mean, \HeadlineMean{}, compared with \CategoricalControlMean{} for
categorical-only and \TrajectoryControlMean{} for trajectory-only training.
This comparison measures the observed specialization pattern across the
single-family variants and the full recipe. The training families still
differ in supervised tasks, so the table does not assign the full joint gain
to one isolated mechanism.
\FloatBarrier

\section{Adaptive Access: Temporal Observation Selection}
\label{app:temporal-support}

Spatial queries differ in the history needed to answer them. Current
relations can often be inferred from recent observations, whereas motion
and event order depend on changes over time. Layout and historical
references call for broader coverage. We use this distinction to adapt
observation selection to the temporal scope of each question. The policy
is fixed before evaluation and applied to the same trained model on
OVO-S; OST and UCS retain their original observation protocols.

\subsection{Matching Observations to Query Scope}

Table~\ref{tab:temporal-basis} summarizes three complementary forms of
temporal coverage. Local sampling concentrates on the current state.
Multiscale sampling emphasizes recent changes while retaining earlier
observations. Global sampling covers the complete history available
at the query time.

\begin{table}[!htbp]
\centering
\caption{\textbf{Temporal coverage for different spatial queries.}
Frame budgets are upper limits; short histories and duplicate removal
can yield fewer frames.}
\label{tab:temporal-basis}
\vspace{4pt}
\small
\setlength{\tabcolsep}{5pt}
\renewcommand{\arraystretch}{1.15}
\begin{tabularx}{\textwidth}{
@{}l*{3}{>{\raggedright\arraybackslash}X}@{}
}
\toprule
& \textbf{Local}
& \textbf{Multiscale}
& \textbf{Global}\\
\midrule
Query scope
& Current state
& Motion and event order
& Layout and historical references\\
\addlinespace[3pt]
Coverage
& Recent observations
& Last 30\,s; 30--300\,s; earlier history
& Full observed prefix\\
\addlinespace[3pt]
Sampling
& Dense tail at 4 fps
& 60\% / 30\% / 10\% of the frame budget
& Uniform\\
\addlinespace[3pt]
Frame budget
& 8
& 128
& 128\\
\bottomrule
\end{tabularx}
\end{table}

Selection follows the visible question text. Named historical endpoints
and persistent relative pose take priority and receive global coverage.
Expressions of motion, change, or chronological order receive multiscale
coverage. Explicit questions about the current state receive local
coverage. Static layout and ambiguous temporal scope default to global
coverage. A named historical reference therefore takes precedence over
an incidental mention of the current state.

The selector uses question text alone. Answer choices, reference answers,
benchmark categories, model predictions, and scores are excluded.
Each question is answered once using the selected observations.

\subsection{Sampling the Causal History}

Let $t_q$ be the query time and $f$ the video frame rate. We retain the
evaluated boundary convention
$e=\max(1,\lfloor f t_q\rfloor)$, with frame indices starting at zero.
The benchmark timestamps place this boundary within the video.

\stitle{Local and global coverage.}
Global sampling takes the integer floors of 128 equally spaced
positions on the inclusive interval $[0,e]$, then sorts and removes
repeated indices. Local sampling uses spacing
$d=\max(1,\operatorname{round}(f/4))$ and the index sequence
\begin{equation}
I_{\mathrm{local}}
=
\{\max(0,e-7d),\max(0,e-7d)+d,\ldots,e\},
\label{eq:local-indices}
\end{equation}
retaining the final eight indices if rounding produces an extra sample.

\stitle{Multiscale coverage.}
With budget $B=128$, the most recent 30 seconds receive
$\operatorname{round}(0.6B)$ samples. The preceding interval extending
to 300 seconds before the query receives
$\operatorname{round}(0.3B)$ samples, and the remaining budget covers
the earlier history.

Define $s_{30}=\max(0,e-30f)$ and
$s_{300}=\max(0,e-300f)$. Let
$\operatorname{Lin}(a,b,n)$ return the integer floors of $n$ equally
spaced positions in $[a,b)$. The requested indices are
\begin{equation}
\begin{aligned}
I_{\mathrm{multi}}={}&
\operatorname{Lin}
(s_{30},e,\operatorname{round}(0.6B))\\
&\cup
\operatorname{Lin}
(s_{300},s_{30},\operatorname{round}(0.3B))\\
&\cup
\operatorname{Lin}
(0,s_{300},n_{\mathrm{early}}),
\end{aligned}
\label{eq:multiscale-indices}
\end{equation}
where $n_{\mathrm{early}}$ is the remaining budget. Empty intervals
are skipped and repeated indices are removed. Global and local
sampling include the query boundary; multiscale intervals exclude
their right endpoint. These conventions preserve the sampling
procedure used for the reported results.

\subsection{Contribution of Observation Selection}
\label{app:access-effects}

Adapting temporal coverage benefits both the official model and
\method{}. The official model improves from \BaseOvosUniform{} to
\BaseOvosSupport{}, while \method{} improves from \CsaOvosUniform{}
to \CsaOvosSupport{} under our fixed OVO-S evaluation protocol. Each comparison holds the
model parameters fixed and varies the observations presented
at inference.

The comparison at a similar frame budget further examines temporal
placement. Uniform sampling of 81 positions from the ordered set
of 128 uniform indices obtains \EqualBudgetOvos{}, whereas adaptive
selection obtains \CsaOvosSupport{} using \MeanFrames{} frames on
average. The advantage of \AdaptiveEqualBudgetGain{} points supports
matching temporal coverage to query scope at approximately the same
visual cost. Table~\ref{tab:ovos-system} separates the training and access
effects using the same 12-family mean.

\section{Qualitative Spatial Analysis}
\label{app:qualitative-case}

The first four cases examine three dependencies of native spatial answers:
historical state retrieval, reference-frame transformation, and accumulated
displacement; their scenes come from ScanNet, Matterport3D, and ARKitScenes
\citep{dai2017scannet,chang2017matterport3d,baruch2021arkitscenes}. Chronological and origin-shifted checkpoints receive identical
ordered observations and questions at evaluation, under the comparison defined
in Appendix~\ref{app:origin-shift-construction}. Figures are selected examples
of output differences; they do not estimate the prevalence of errors or
identify an observation-selection failure. Figures~\ref{fig:spatial-first-appearance}--\ref{fig:spatial-motion-distance}
use Run~2 of Table~\ref{tab:origin-seed-details}.

Two additional cases examine native supervision complementarity and a global-layout
failure. The first uses the historical diagnostic models specified in
Appendix~\ref{app:supervision-case}; the second compares Base and the headline
\method{} model under fixed observations.

\subsection{Historical State Retrieval}
\label{app:historical-retrieval}

Figure~\ref{fig:spatial-case} illustrates how an earlier state supports both a
temporal answer and a later reference-relative displacement. The retrieval
question requires the round in which the two objects occupied the specified
side of the observer. The outputs distinguish recovery of that earlier state
from selection of a later observation.

\begin{figure}[!htbp]
  \centering
  \includegraphics[
    width=\linewidth,
    height=0.90\textheight,
    keepaspectratio
  ]{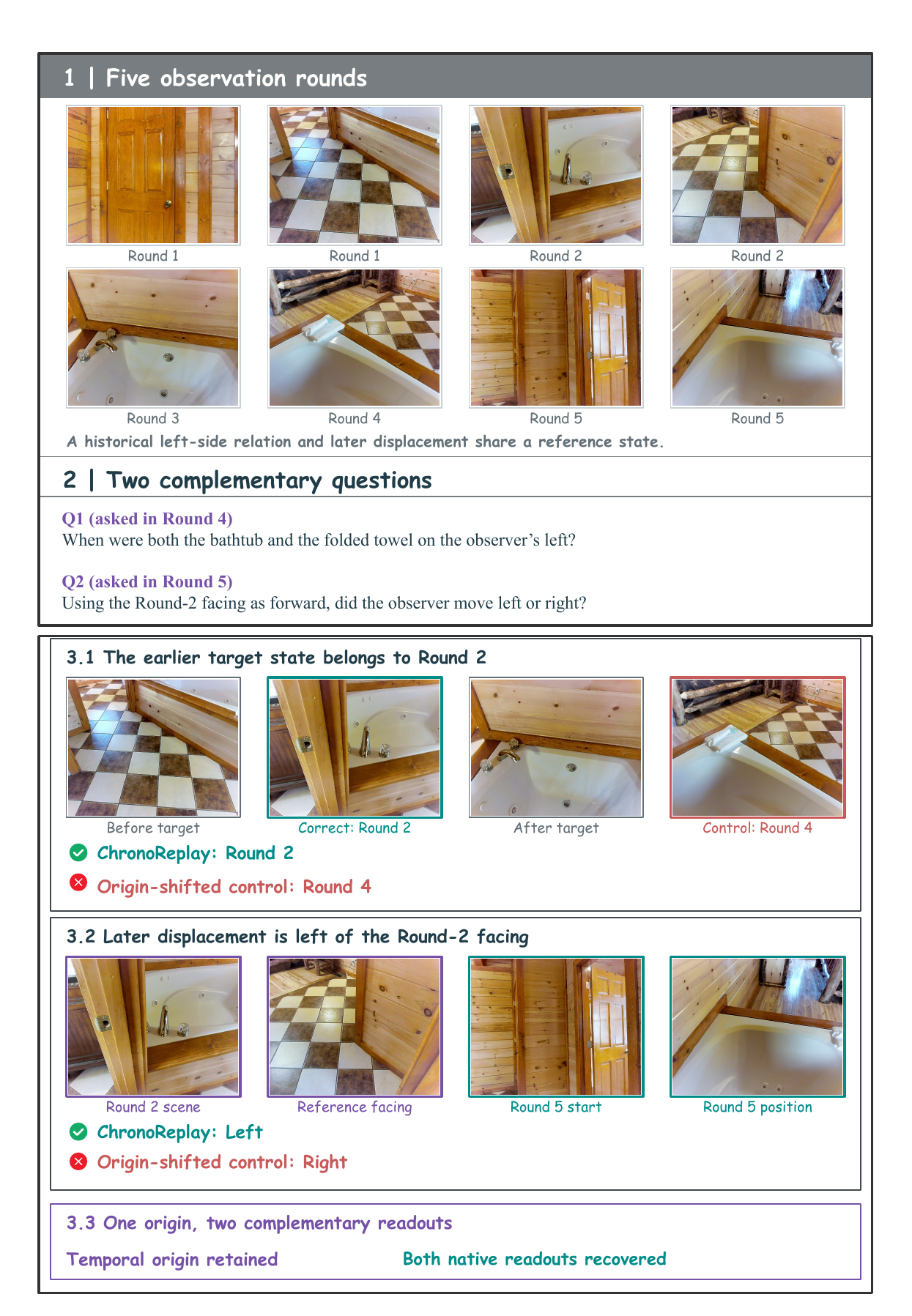}
  \caption{\textbf{Historical context in spatial reasoning.}
  The first question retrieves an earlier state in which the bathtub
  and folded towel lie on the observer's left. The second evaluates
  a later location relative to the facing direction in Round~2.
  Teal indicates annotated states, and red shows the control's
  incorrect retrieval or displacement. Both questions depend on
  a reference established earlier in the episode.}
  \label{fig:spatial-case}
\end{figure}

Figure~\ref{fig:spatial-first-appearance} examines the related distinction
between first appearance and category order. The queries require identifying
which object was observed and when it first appeared. The apparent early
match is a different fixture, making object identity part of the temporal
retrieval problem.

\begin{figure}[!htbp]
  \centering
  \includegraphics[height=0.90\textheight,width=\linewidth,keepaspectratio]{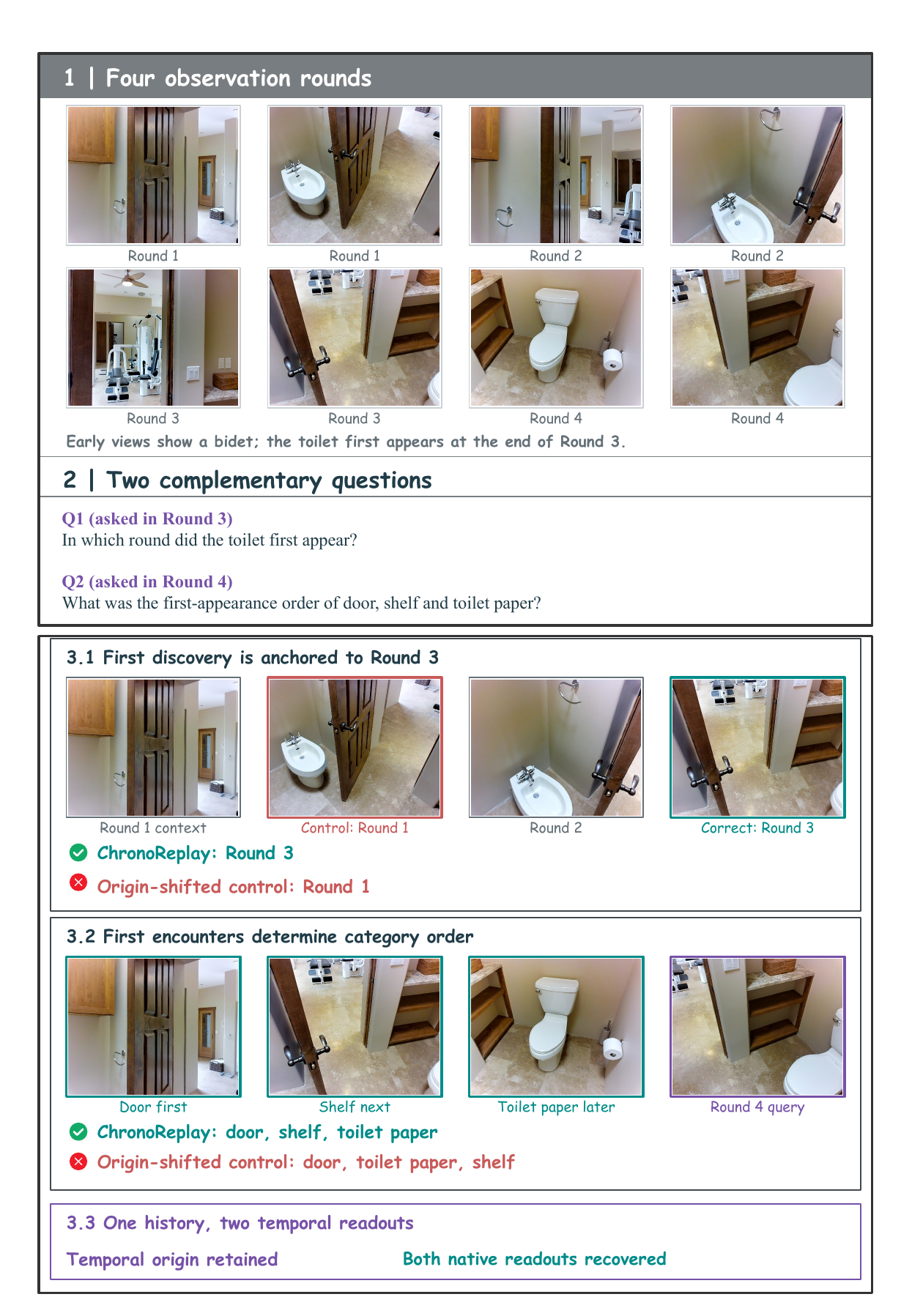}
  \caption{\textbf{First appearance and category order.} In scene
  \texttt{1mp3d\_0067\_region16}, ChronoReplay returns Round~3 for the first
  appearance of the toilet and the order door, shelf, toilet paper. The matched
  origin-shifted control returns Round~1 and reverses the latter two categories.
  The early white fixture is a bidet; the toilet is partially visible at the
  lower-right edge of the final Round~3 frame.}
  \label{fig:spatial-first-appearance}
\end{figure}

\subsection{Reference-Frame Transformation}
\label{app:reference-frame-cases}

Figure~\ref{fig:spatial-reference-rotation} combines a historical co-visibility
query with a rotation query anchored to an earlier facing direction. The
second answer depends on interpreting later observations in the specified
reference frame. Its sign error is distinct from returning the wrong round
for the first question, although both queries depend on earlier observations.

\begin{figure}[!htbp]
  \centering
  \includegraphics[height=0.90\textheight,width=\linewidth,keepaspectratio]{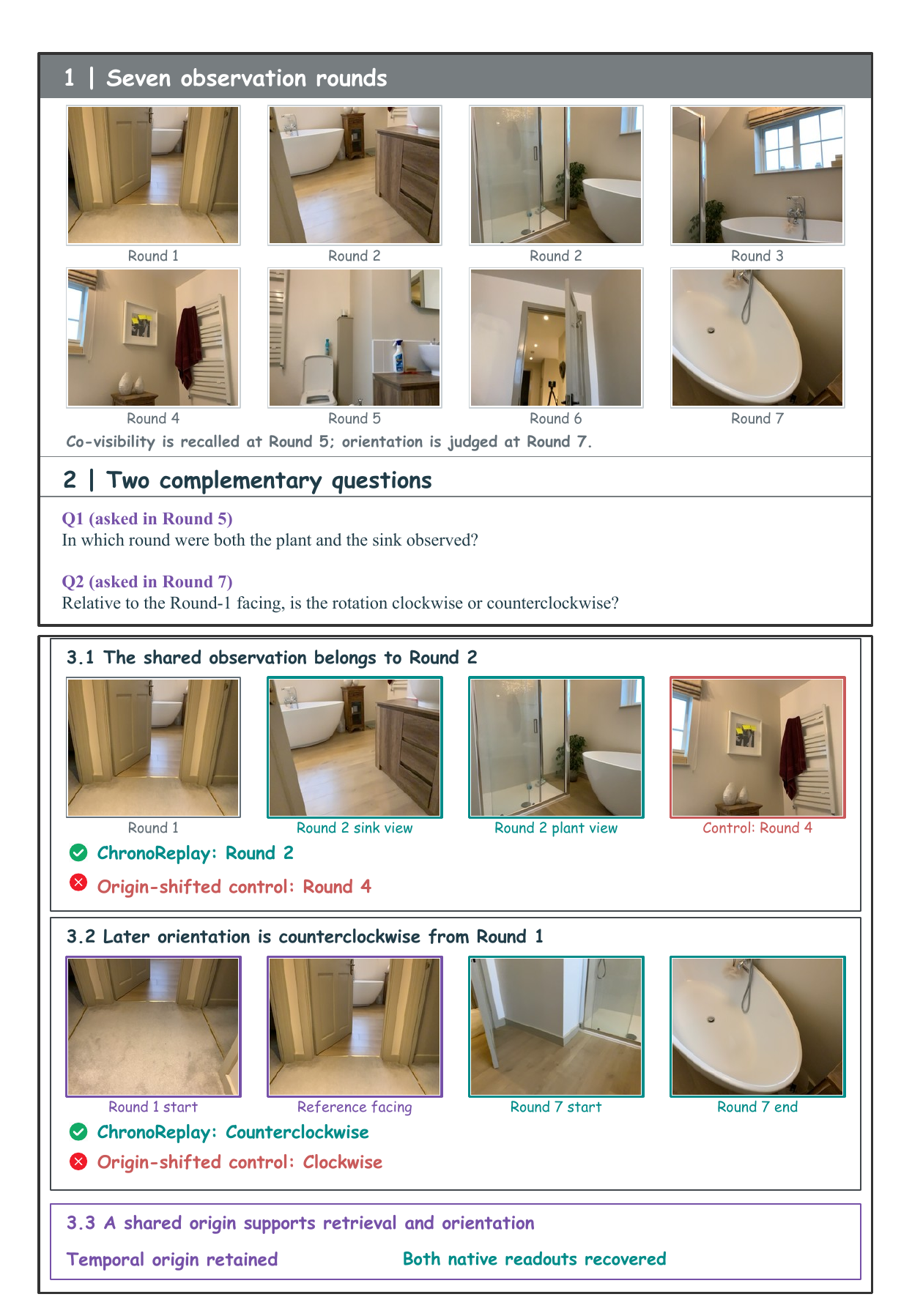}
  \caption{\textbf{Historical retrieval and rotation.}
  In scene \texttt{Validation\_42898818}, the Round~5 question asks which
  round contained both the plant and sink; ChronoReplay returns Round~2,
  whereas the control returns Round~4. A Round~7 question defines the
  Round~1 facing as the orientation reference and specifies a rotation below
  $180^\circ$. ChronoReplay correctly returns counterclockwise, while the
  control returns clockwise.}
  \label{fig:spatial-reference-rotation}
\end{figure}

\begin{samepage}
\subsection{Accumulated Displacement and Historical Proximity}
\label{app:displacement-cases}

Figure~\ref{fig:spatial-motion-distance} asks about displacement relative to a
reference facing and about the closest previously observed state. The first
requires tracking a location change, rather than reporting turning direction;
the second requires comparing historical distances using the source annotation.
The paired outputs illustrate two ways an answer can disagree with its
historical reference under the same evaluation observations.
\end{samepage}

\begin{figure}[!htbp]
  \centering
  \includegraphics[height=0.87\textheight,width=\linewidth,keepaspectratio]{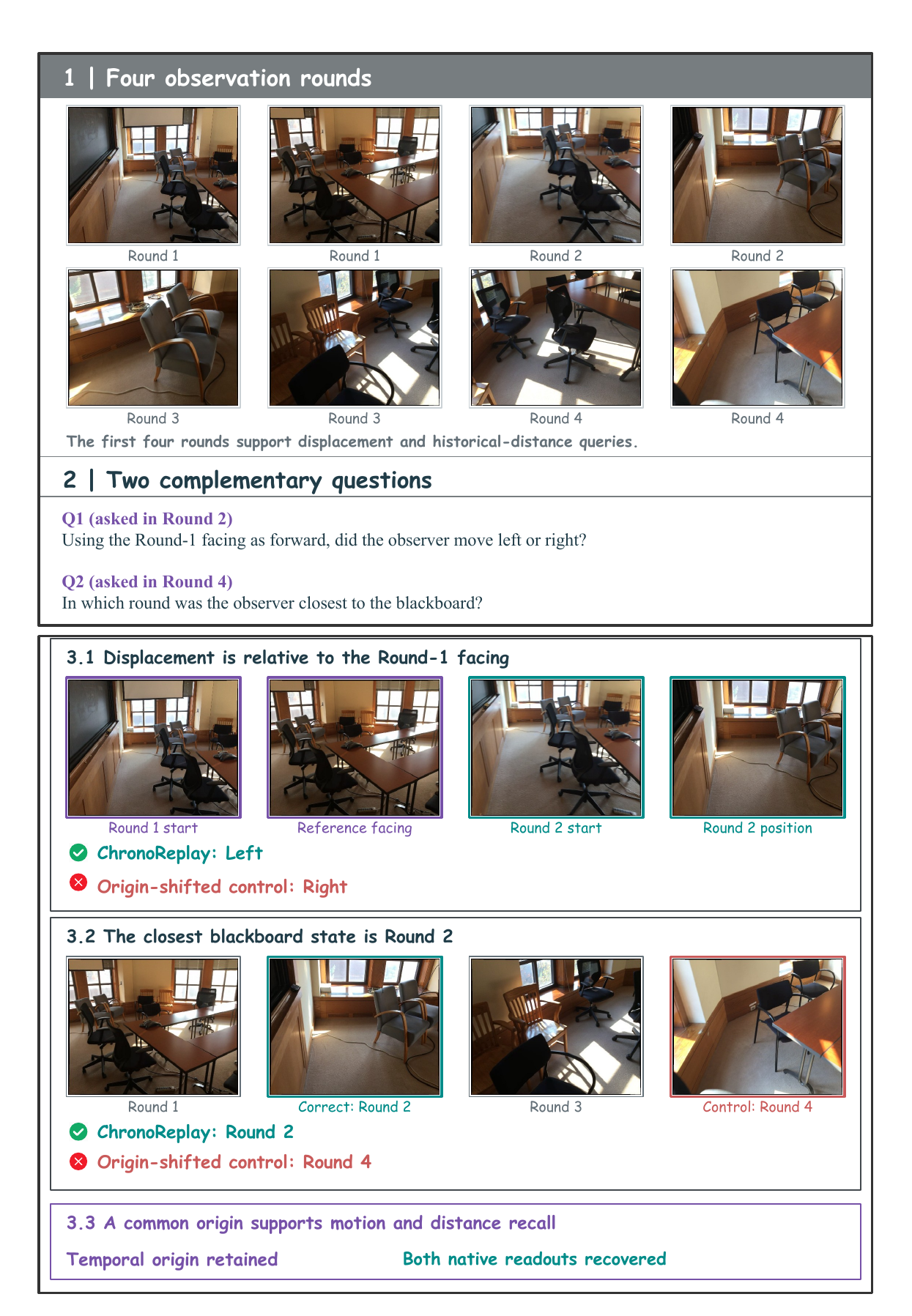}
  \caption{\textbf{Reference displacement and nearest-state retrieval.}
  In scene \texttt{scene0496\_00}, ChronoReplay identifies leftward displacement
  relative to the Round~1 facing, while the control answers right. When asked
  in Round~4 for the closest blackboard state, ChronoReplay retrieves Round~2
  rather than the control's Round~4. Left/right denotes displacement, not
  turning direction; the closest-state label follows the source annotation
  rather than apparent object size in a single image.}
  \label{fig:spatial-motion-distance}
\end{figure}

\clearpage
\subsection{Native Numerical Estimation and Supervision Complementarity}
\label{app:supervision-case}

Figure~\ref{fig:supervision-complementarity} compares Base, historical
Preservation and Trajectory diagnostic models, and Joint replay on two
questions about the same ornate
mirror in OST-Bench scene \texttt{1mp3d\_0004\_region16}. The numerical
question is asked at the end of Round~3 and the categorical question at the
end of Round~4; this is a shared episode with different query endpoints,
not two questions conditioned on an identical complete history. All four
settings receive the same source-ordered images and question prompts.
The historical Preservation and Trajectory models differ from the
single-family variants in Table~\ref{tab:controlled-analysis}(a).
As in the official multi-round protocol, each model retains its own earlier
responses. The Round~4 observations are not exposed to the Round~3 query.

\stitle{Different strengths, rather than a uniform winner.}
The source annotation places the mirror on the observer's right in Round~4.
Base, Preservation, and Joint replay answer right, whereas Trajectory answers
left. The Round~3 horizontal distance is annotated as 0.95\,m; the four
predictions are 1.20, 2.00, 1.10, and 1.11\,m, respectively. Trajectory has
the smallest absolute error, 0.15\,m, followed by Joint at 0.16\,m. Thus
Joint retains the correct categorical answer while approaching the numerical
accuracy of Trajectory; it is not best on every individual readout.

\stitle{Official error convention.}
For this nonzero scalar target $y$, OST uses relative error
$\delta=|\hat y-y|/|y|$ and awards the fraction of ten strict thresholds
satisfied:
\begin{equation}
s_{\mathrm{est}}=\frac{1}{10}
\sum_{i=0}^{9}\mathbf{1}\!\left[\delta<0.50-0.05i\right].
\label{eq:case-ost-estimation}
\end{equation}
The respective relative errors are 26.3\%, 110.5\%, 15.8\%, and 16.8\%,
giving official scores 0.50, 0.00, 0.70, and 0.70. Trajectory and Joint tie
under this discretized metric despite their different raw errors. Distances
and directions follow the source annotation, not measurements inferred from
RGB appearance. This selected episode illustrates the aggregate trade-off;
it does not estimate its frequency.

\subsection{A Global Layout Regression under Fixed Observations}
\label{app:global-layout-failure}

Figure~\ref{fig:global-layout-failure} shows OVO-S L4.3 record
\texttt{4.3\_034}, an ARKitScenes room tour ending at 50 seconds. The question
asks which supplied bird's-eye view matches the scene. Under the same
Uniform-128 observations, the official Base selects the annotated option B,
while the headline ChronoReplay checkpoint selects D. Both archived records
use the same default prompt, ordered image blocks, greedy categorical
generation, and unpermuted option labels. This pair isolates an output
regression under fixed access; it is not an adaptive-versus-uniform comparison.

\stitle{The concrete disagreement.}
The video reveals the fireplace, television, table, and surrounding furniture
across successive views. In the shared coordinate frame of the supplied
option diagrams, B places the fireplace and TV in the lower region and the
table across the layout. D instead places the fireplace in the upper region
and moves the table beside the TV in the lower-right region. The figure pairs
key video observations with the original B and D option diagrams, labelled by
the model that selected each option. All four layouts are retained in the
supplementary source assets. The source provides
no evidence-time intervals for this item, so the displayed frames are selected
from the recorded sampling protocol, not labelled as official evidence spans.
This discrepancy demonstrates a global-layout failure. It does not establish
that the model built and then lost a global map, nor isolate perception,
cross-view integration, or option matching as the internal cause.

Exact questions, model paths, responses, numerical scores, frame references,
and all four source layout options accompany the manuscript in
\texttt{supplementary/spatial\_capability\_cases.json}. No new model inference
is used for either case.

\clearpage
\begin{figure}[H]
  \centering
  \includegraphics[height=0.87\textheight,width=\linewidth,keepaspectratio]{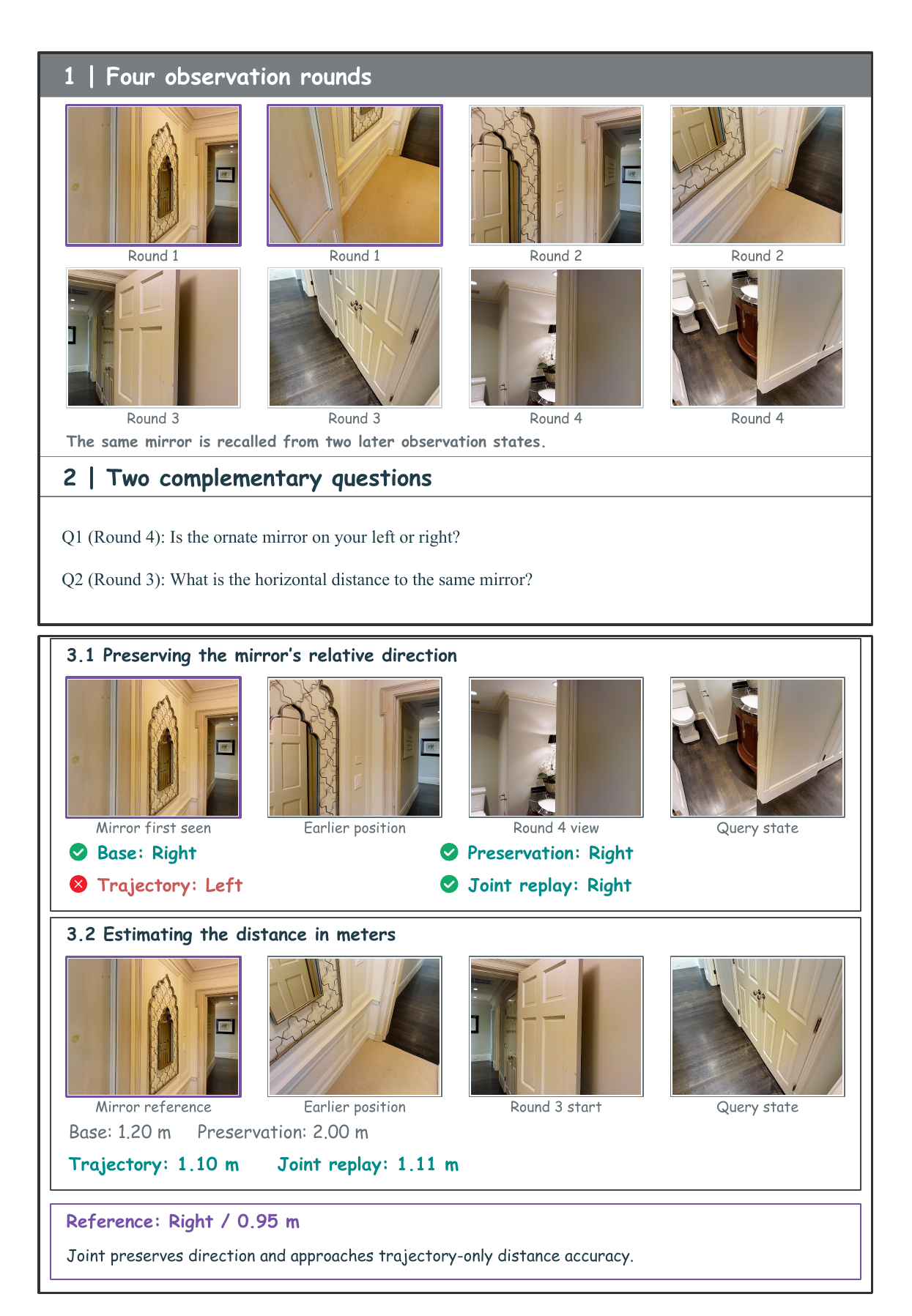}
  \vspace{-0.4em}
  \caption{\textbf{Two native interfaces on one mirror history.}
  Selected chronological observations precede a left/right judgement and a
  horizontal-distance estimate. Round labels identify the source observation turns. Base, Preservation, Trajectory, and Joint replay
  use the same observations at each query endpoint. Joint retains the correct
  direction; Trajectory has slightly lower absolute distance error, and both
  tie at 0.70 under the official metric.}
  \label{fig:supervision-complementarity}
\end{figure}

\clearpage
\begin{figure}[H]
  \centering
  \includegraphics[width=\linewidth,trim=0bp 452bp 0bp 0bp,clip]{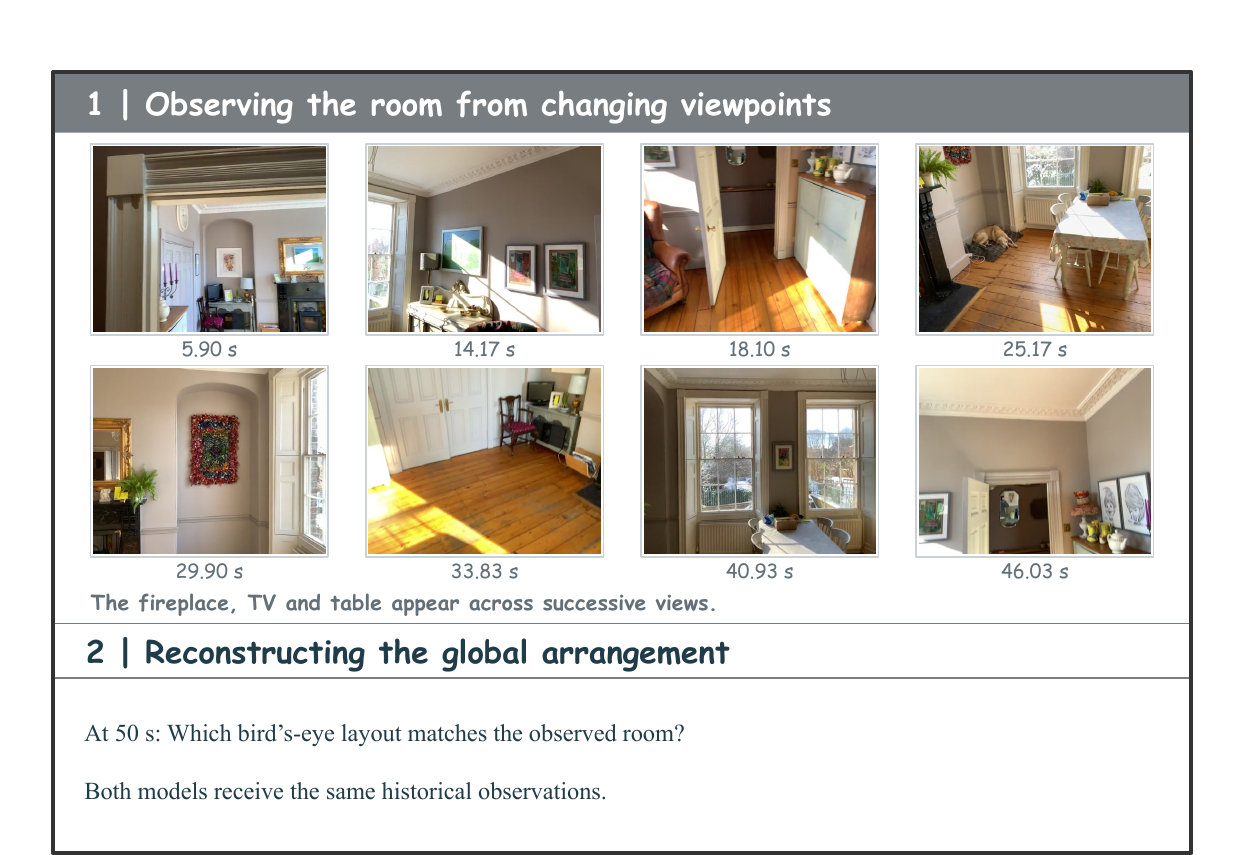}
  \par\vspace{6pt}
  \begin{minipage}[t]{0.48\linewidth}
    \centering
    \small\textbf{\textcolor{teal}{B: annotated layout / Base}}
    \par\vspace{3pt}
    \setlength{\unitlength}{\linewidth}
    \begin{picture}(1,0.8)
      \put(0,0){\includegraphics[width=\linewidth]{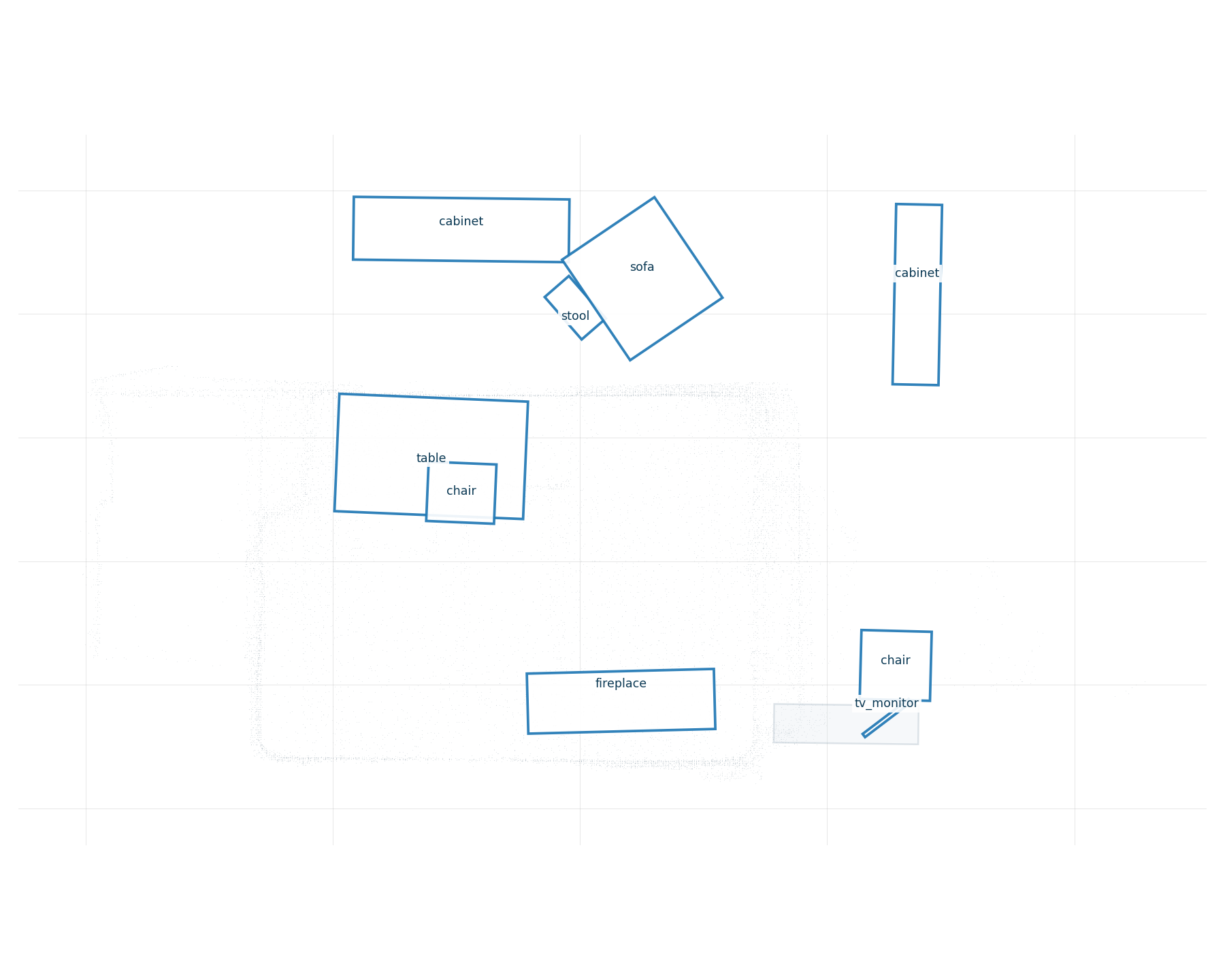}}
      \put(0.34,0.37){\makebox(0,0){\sffamily\scriptsize\bfseries Table}}
      \put(0.51,0.17){\makebox(0,0){\sffamily\scriptsize\bfseries Fireplace}}
      \put(0.74,0.14){\makebox(0,0){\sffamily\scriptsize\bfseries TV}}
    \end{picture}
    \par\small Fireplace and TV below the table
  \end{minipage}\hfill
  \begin{minipage}[t]{0.48\linewidth}
    \centering
    \small\textbf{\textcolor{red}{D: ChronoReplay's choice}}
    \par\vspace{3pt}
    \setlength{\unitlength}{\linewidth}
    \begin{picture}(1,0.8)
      \put(0,0){\includegraphics[width=\linewidth]{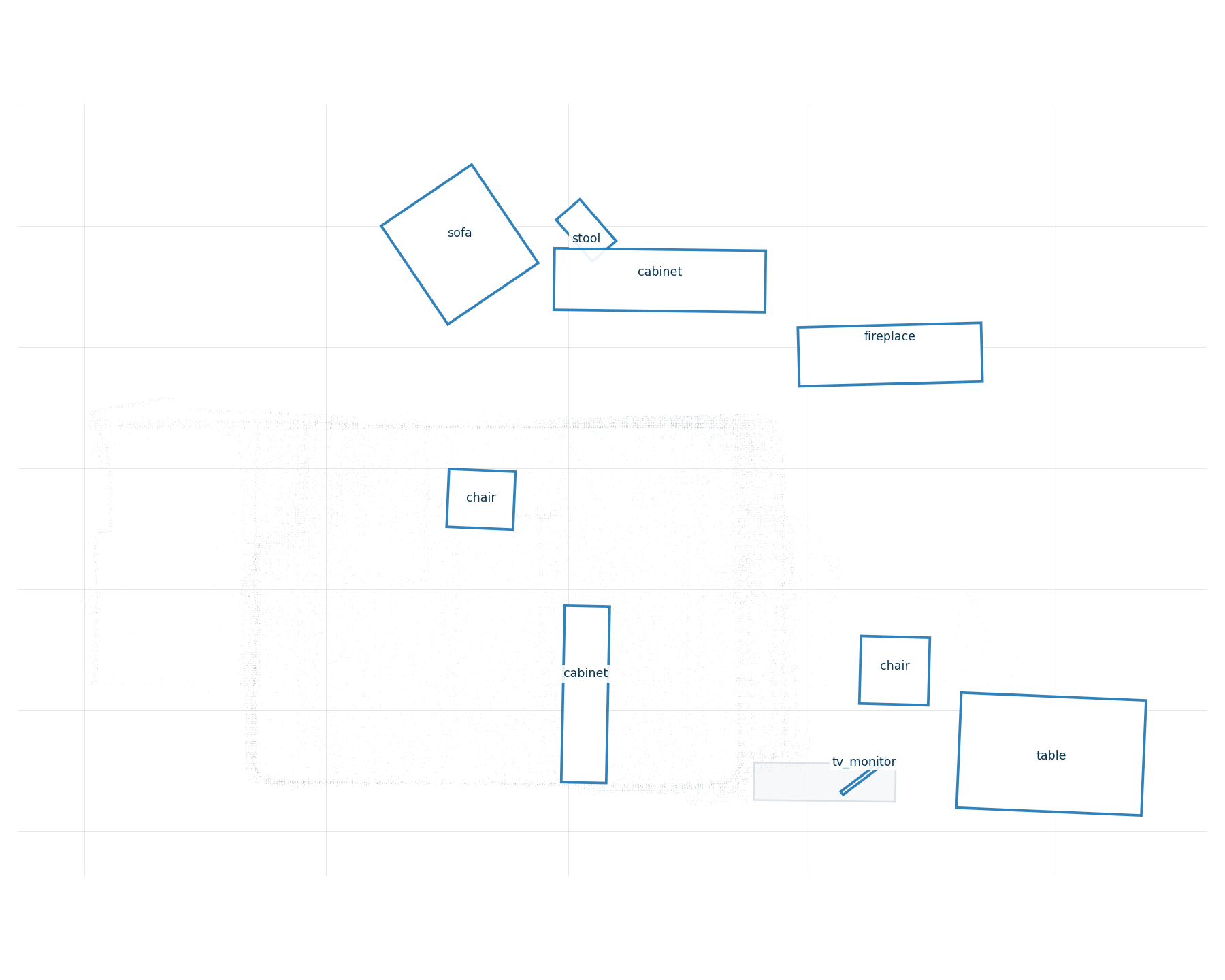}}
      \put(0.74,0.455){\makebox(0,0){\sffamily\scriptsize\bfseries Fireplace}}
      \put(0.715,0.095){\makebox(0,0){\sffamily\scriptsize\bfseries TV}}
      \put(0.87,0.055){\makebox(0,0){\sffamily\scriptsize\bfseries Table}}
    \end{picture}
    \par\small Fireplace above TV; table right of TV
  \end{minipage}
  \caption{\textbf{A fixed-observation global-layout failure.}
  The official Base chooses B and ChronoReplay chooses D on OVO-S L4.3.
  The upper panel shows eight timestamped frames from their shared Uniform-128
  input. The lower panels reproduce the two source option diagrams in their
  original orientations, with enlarged labels beside the key objects.
  Captions describe positions within these diagrams.
  The fireplace and table placements make the disagreement visible.}
  \label{fig:global-layout-failure}
\end{figure}

\clearpage

\section{Supplementary Benchmark Results}
\label{app:full-results}
\label{app:full-benchmark-tables}

\subsection{OVO-S-Bench}

Table~\ref{tab:ovos-main} reports the complete twelve-family decomposition
for the official base and \method{} under Uniform-128 and adaptive access.
Table~\ref{tab:ovos-system} reports the corresponding level means.
The independent Uniform-81 control is in Appendix~\ref{app:access-effects}. The supplementary table below
provides additional published aggregate reference scores; the evaluation
population and comparison scope are defined in Appendix~\ref{app:protocol}.

\begin{table}[htbp]
\caption{Additional published OVO-S aggregate reference rows (\%).  These
rows provide external context and are not used for matched training
attribution.}
\label{tab:ovos-published-full}
\centering
\footnotesize
\begin{tabular}{lcccccc}
\toprule
Model & Params & L1 & L2 & L3 & L4 & Overall\\
\midrule
\rowcolor{GroupGray}
\multicolumn{7}{l}{\textit{Proprietary MLLMs}}\\
GPT-5.4 & -- & 54.64 & 57.59 & 50.80 & 40.53 & 50.89\\
Gemini-3.1-Pro & -- & 61.92 & 64.04 & 55.90 & 54.90 & 59.19\\
Gemini-3.1-Flash-Lite & -- & 54.09 & 52.19 & 54.14 & 42.84 & 50.81\\
Grok-4.1-Fast & -- & 44.80 & 46.59 & 48.49 & 35.04 & 43.73\\
\midrule
\rowcolor{GroupGray}
\multicolumn{7}{l}{\textit{Open-source MLLMs}}\\
InternVL-3.5 & 38B & 54.65 & 54.40 & 45.53 & 41.94 & 49.13\\
InternVL-3.5 & 241B-A28B & 55.55 & 55.72 & 51.62 & 40.54 & 50.85\\
Qwen3-VL & 4B & 43.25 & 48.18 & 54.54 & 41.35 & 46.83\\
Qwen3-VL & 32B & 50.14 & 51.91 & 51.94 & 41.22 & 48.80\\
Qwen3-VL & 235B-A22B & 52.45 & 55.18 & 61.20 & 45.71 & 53.64\\
\bottomrule
\end{tabular}
\end{table}
\FloatBarrier

\subsection{OST-Bench}

Table~\ref{tab:ost-full} retains the aggregate OST summary and the complete
published reference set used in this paper.  The same native scorer and
scene-disjoint manifest are used for the official base and ChronoReplay rows.

\begin{table}[htbp]
\caption{Aggregate OST-Bench summary and published reference set (\%).
Published rows use the released full split; the final block is compared only
within the matched scene-disjoint protocol.}
\label{tab:ost-full}
\centering
\footnotesize
\setlength{\tabcolsep}{3.2pt}
\resizebox{\textwidth}{!}{%
\begin{tabular}{lcccccccc}
\toprule
Model & Overall & A. State & A. Info & AO & Judgement & Temporal & Counting & Estimation \\
\midrule
\rowcolor{GroupGray}
\multicolumn{9}{l}{\textit{Proprietary MLLMs (published)}}\\
Claude-3.5-Sonnet & 47.8 & 45.6 & 65.6 & 32.9 & 57.4 & 46.7 & 57.9 & 26.9\\
Gemini-2.0-Flash & 49.5 & 45.1 & 70.8 & 33.8 & 61.1 & 45.8 & 59.4 & 28.1\\
Gemini-2.0-Flash (Thinking) & 54.2 & 47.1 & 72.3 & 42.8 & 63.6 & 61.1 & 61.9 & 28.9\\
GPT-4o & 48.7 & 38.8 & 72.8 & 33.5 & 59.8 & 48.6 & 59.8 & 24.0\\
GPT-4.1 & 53.4 & 47.2 & 76.5 & 37.7 & 66.4 & 51.7 & 60.6 & 29.1\\
\midrule
\rowcolor{GroupGray}
\multicolumn{9}{l}{\textit{Open-source MLLMs (published)}}\\
InternVL-2.5-8B & 44.6 & 41.9 & 58.6 & 34.6 & 52.4 & 42.7 & 56.4 & 29.3\\
InternVL-2.5-38B & 50.8 & 45.4 & 73.9 & 34.0 & 61.4 & 45.9 & 61.1 & 31.5\\
InternVL-2.5-78B & 51.1 & 46.5 & 74.0 & 32.9 & 61.5 & 47.1 & 65.9 & 29.0\\
Qwen2.5-VL-7B & 41.2 & 40.4 & 52.6 & 31.5 & 50.1 & 36.6 & 62.1 & 22.7\\
Qwen2.5-VL-72B & 45.6 & 43.5 & 64.5 & 28.9 & 55.9 & 41.6 & 61.5 & 22.9\\
\midrule
\rowcolor{GroupGray}
\multicolumn{9}{l}{\textit{Matched scene-disjoint evaluation}}\\
Official base & \BaseOst{} & \BaseOstAState{} & \BaseOstAInfo{}
& \BaseOstAO{} & \BaseOstJudge{} & \BaseOstTemporal{}
& \BaseOstCount{} & \BaseOstEst{} \\
\rowcolor{MethodBlue}
ChronoReplay & \CsaOst{} & \CsaOstAState{} & \CsaOstAInfo{}
& \CsaOstAO{} & \CsaOstJudge{} & \CsaOstTemporal{}
& \CsaOstCount{} & \CsaOstEst{} \\
\bottomrule
\end{tabular}}
\end{table}
\FloatBarrier

The gain spans every reported component, including continuous Estimation and
Temporal-localization.  The matched origin experiment is reported separately
in the temporal-origin block of Table~\ref{tab:controlled-analysis}; it does
not replace this checkpoint-level
comparison.

\subsection{UCS-Bench}

Table~\ref{tab:ucs-main} reports four category aggregates and eight UCS subcategories over
the complete 8,114-question evaluation. The table below adds the aggregate
published reference scores and matched model comparison. Appendix~\ref{app:protocol}
defines causal access and scoring scope. Training-topology robustness is
evaluated separately in Appendix~\ref{app:topology-replication}.

\begin{table}[htbp]
\caption{Published UCS aggregate reference set and matched 8,114-question
evaluation (\%).  Published rows use their reported frame protocols.}
\label{tab:ucs-published-full}
\centering
\footnotesize
\resizebox{\textwidth}{!}{%
\begin{tabular}{lccccccc}
\toprule
Model & Size & Frames & Trajectory & Position & Proximity & Category & Overall\\
\midrule
\rowcolor{GroupGray}
\multicolumn{8}{l}{\textit{Human and proprietary references}}\\
Human & -- & -- & 91.7 & 87.6 & 96.8 & 90.0 & 91.0\\
DeepSeek-V3 & -- & 50 & 42.0 & 33.2 & 37.7 & 32.3 & 35.6\\
Doubao-Seed-1.6V & -- & 50 & 49.6 & 42.8 & 56.0 & 44.6 & 47.3\\
GPT-5 & -- & 50 & 44.5 & 37.8 & 55.4 & 41.4 & 43.7\\
Gemini-3-Flash & -- & 50 & 43.8 & 38.0 & 52.9 & 38.9 & 42.2\\
\midrule
\rowcolor{GroupGray}
\multicolumn{8}{l}{\textit{Open-source MLLMs}}\\
InternVideo2.5 & 8B & 64 & 44.2 & 39.4 & 52.7 & 42.2 & 43.8\\
Qwen2.5-VL & 7B & 64 & 41.6 & 40.8 & 51.3 & 42.5 & 43.4\\
Qwen2.5-VL & 32B & 64 & 50.3 & 40.9 & 55.4 & 46.0 & 47.3\\
Qwen3-VL & 8B & 64 & 44.6 & 40.1 & 50.9 & 42.9 & 44.0\\
Qwen3-VL & 32B & 64 & 50.3 & 46.5 & 55.2 & 46.7 & 49.0\\
DirectMe (Qwen3-VL) & 8B & 1fps & 54.1 & 47.6 & 58.0 & 46.4 & 50.5\\
\midrule
\rowcolor{GroupGray}
\multicolumn{8}{l}{\textit{Matched evaluation}}\\
Official Qwen3-VL-8B & 8B & 128 & \BaseUcsTrajectory{} & \BaseUcsEgo{}
& \BaseUcsProximity{} & \BaseUcsObject{} & \BaseUcs{}\\
\rowcolor{MethodBlue}
ChronoReplay & 8B & 128 & \CsaUcsTrajectory{} & \CsaUcsEgo{}
& \CsaUcsProximity{} & \CsaUcsObject{} & \CsaUcs{}\\
\bottomrule
\end{tabular}}
\end{table}
\FloatBarrier

\clearpage

\section{Relation-Supervision Diagnostics}
\label{app:relation-diagnostics}
\label{app:compatibility-controls}
\label{app:relation-algebra}

This appendix reports a negative diagnostic: whether adding a more
discriminative auxiliary relation interface, a binary compatibility judgement
between a causal state and a typed relation claim, yields a more reusable
native representation.  These variants do not contribute to the headline
\method{} checkpoint or its training data.  Relation rows are compiled from
verified source relations (signed heading, metric displacement, temporal
retrieval, counting, topology, and local state) paired with controlled
incompatible alternatives that stay in the same value domain; positive and
incompatible claims are independent continuations of one observation prefix,
and only the final compatibility token receives loss.  Source question,
candidate order, option letter, benchmark identity, and measured score are
absent from the state prefix.  Within-state normalization assigns fixed mass
to each causal state regardless of the number of relations $K$; the local
$K=1$ control adds \ProbeKOneRows{} rows over \ProbeKOneStates{} states, and
cross-type $K=4$ adds \FinalProbeRows{} rows over \FinalProbeStates{} states.

\begin{table}[htbp]
\caption{Overview of relation-supervision diagnostics.  Native columns are
complete model-only benchmark scores; Readout is the matched withheld
compatibility diagnostic (\%).}
\label{tab:relation-variant-overview}
\centering
\small
\begin{tabular}{lcccc}
\toprule
Variant & OVO-S & OST & UCS & Readout \\
\midrule
\method{}, no relation loss & \CsaOvosUniform{} & \CsaOst{} & \CsaUcs{} & -- \\
Local compatibility, $K=1$ & \LocalCompatOvosUniform{} & \LocalCompatOst{}
& \LocalCompatUcs{} & \LocalCompatAccuracy{} \\
Same type, $K=4$ & \SameTypeKFourOvos{} & \SameTypeKFourOst{}
& \SameTypeKFourUcs{} & \SameTypeKFourReadout{} \\
Cross type, $K=4$ & \CrossTypeKFourOvos{} & \CrossTypeKFourOst{}
& \CrossTypeKFourUcs{} & \CrossTypeKFourReadout{} \\
State-path firewall & \StateWriteOvosUniform{} & \StateWriteOst{}
& \StateWriteUcs{} & \StateWriteReadout{} \\
\bottomrule
\end{tabular}
\end{table}
\FloatBarrier

Cross-type supervision improves the withheld relation readout over the matched
same-type control, but its native three-endpoint mean is lower.  The
state-path firewall, which restricts relation gradients to the visual-state
route, raises UCS while reducing OVO-S and OST.  Formulation controls on
complete OST and the 480-row withheld readout behave the same way: removing
explicit type names gives \NoTypeOst{} OST and \NoTypeReadout{} readout,
presenting candidates jointly gives \JointChoiceOst{} and
\JointChoiceReadout{}, and removing state normalization gives
\UnnormalizedStateMassOst{} and \UnnormalizedStateMassReadout{}, so each
moves the local readout without improving native transfer.  An
unbalanced-polarity variant reaches \UnbalancedN{}\% on incompatible claims
but only \UnbalancedY{}\% on compatible ones (\UnbalancedOst{} OST).
Auxiliary separability is therefore not sufficient for native cross-interface
transfer, which the main paper evaluates through the native benchmarks.

\section{Reproducibility Details}
\label{app:reproducibility}

\subsection{Executable Configuration}

All runs start from the pinned \texttt{Qwen/Qwen3-VL-8B-Instruct} revision
listed in Table~\ref{tab:execution-config} and reuse its tokenizer, processor,
and chat template.  Training follows Table~\ref{tab:training-config}, without
adapters, continuation, fusion, curriculum, RL, or offload.

\begin{table}[H]
\caption{Execution settings for the headline, target-format, and chronology
comparisons in Appendix~\ref{app:experimental-details}.}
\label{tab:execution-config}
\centering
\small
\begin{tabular*}{\textwidth}{@{\extracolsep{\fill}}ll}
\toprule
Setting & Value\\
\midrule
Base-model revision
& \path{0c351dd01ed87e9c1b53cbc748cba10e6187ff3b}\\
Headline / target-format seed & 42\\
Chronology seeds & 42 (Run 1); 73 (Run 2)\\
Headline and chronology topology & 8 GPUs; gradient accumulation 2\\
Target-interface topology & 4 GPUs; gradient accumulation 4\\
\bottomrule
\end{tabular*}
\end{table}

Inference is deterministic. OVO-S uses greedy answer generation with the
fixed evaluation prompt and deterministic answer extraction, an eight-token
ceiling, and batch size two. OST uses greedy decoding with
temperature $0$, top-$p=1$, one sample, and 512 tokens; UCS uses temperature
$0$, one sample, and eight tokens.  No voting, prompt ensemble, or
benchmark-dependent parameter choice contributes to a paper row.

\subsection{Data and Media Audits}

All training examples are bound to canonical scene identities and source
annotations before optimization.  The same data separation is applied to the
main model and every matched control.  Complete media, row, and scene manifests
are retained with the released configuration.

Each release binds row hashes, counts, family mass, supervised turns, and the
media manifest.  Every referenced image and video is decoded and audited once
before training: items at or above the 32,768-token limit or above the
declared visual-pixel budget are rejected, and isolation audits bind canonical
physical scenes.

\subsection{Scorer and Manifest Contracts}

Each reported result binds the model, observations, predictions, annotations,
and scoring protocol. Evaluation follows the fixed OVO-S protocol in
Appendix~\ref{app:protocol}, the fixed 8,351-turn OST manifest, and all
8,114 UCS questions. Result
records and source manifests link every evaluation identity to the values
defined in \path{main.tex} and will be released with the code.

Native responses are passed to the unchanged task scorer.  For canonical
responses, a frozen reader extracts the leading typed \texttt{value} field
from the generated response alone and supplies that value to the same scorer;
it does not inspect the question, task identity, reference answer, label, or
score.  All 8,351 OST responses and all 8,114 UCS responses yield an extracted
value.  The official OST scorer evaluates its designated 7,814 turns.